\documentclass[11pt]{article}
\usepackage{fullpage}
\usepackage{amsfonts}
\usepackage{amssymb}
\usepackage{fancyhdr}
\usepackage{comment}
\usepackage{xspace}
\usepackage{tikz}
\usepackage{float}
\usepackage[title]{appendix}
\usepackage{caption}
\usepackage{subcaption}

\usepackage[margin=1in]{geometry}  
\usepackage[
            CJKbookmarks=true,
            bookmarksnumbered=true,
            bookmarksopen=true,
            colorlinks=true,
            citecolor=red,
            linkcolor=blue,
            anchorcolor=red,
            urlcolor=blue,
            ]{hyperref}

\usepackage{amsmath}
\usepackage{amsthm}
\usepackage{dsfont}
\usepackage{graphicx}
\usepackage{enumerate}
\usepackage[hypcap]{caption}
\usepackage{color, soul}
\usepackage{xcolor}
\usepackage{float}
\usepackage{algorithm}
\usepackage[noend]{algorithmic}
\newtheorem{theorem}{Theorem}
\newtheorem{remark}{Remark}
\newtheorem{lemma}{Lemma}

\newtheorem{assumption}{Assumption}
\newtheorem{definition}{Definition}
\newtheorem{proposition}{Proposition}

\newcommand{\red}{\color{red}}
\newcommand{\blue}{\color{blue}}
\newcommand{\violet}{\color{violet}}
\newcommand{\nb}[1]{{\sf\blue[#1]}}
\newcommand{\nbr}[1]{{\sf\red[#1]}}
\newcommand{\nbv}[1]{{\sf\violet[#1]}}

\newcommand{\indep}{\rotatebox[origin=c]{90}{$\models$}}

\def\kxb#1 {{\color{blue} [(KX) #1]}}
\def\kxr#1 {{\color{red} [(KX)  #1]}}

\newcommand{\Dir}{\mathrm{Dir}}
\newcommand{\Beta}{\mathrm{Beta}}

\makeatletter
\newtheorem*{rep@theorem}{\rep@title}
\newcommand{\newreptheorem}[2]{%
\newenvironment{rep#1}[1]{%
 \def\rep@title{#2 \ref{##1}}%
 \begin{rep@theorem}}%
 {\end{rep@theorem}}}
\makeatother

\newreptheorem{lemma}{Lemma}

\let\svthefootnote\thefootnote
\newcommand\freefootnote[1]{%
  \let\thefootnote\relax%
  \footnotetext{#1}%
  \let\thefootnote\svthefootnote%
}

\def \N{\mathbb{N}}
\usepackage{kx_macros}

\title{Learner-Private Convex Optimization}

\author{Jiaming Xu, Kuang Xu, and Dana Yang\thanks{
J.\ Xu and D.\ Yang are with The Fuqua School of Business, Duke University, Durham NC, USA, \texttt{\{jx77,xiaoqian.yang\}@duke.edu}.
K.\ Xu is with the Stanford Graduate School of Business, Stanford University, Stanford CA, USA, \texttt{kuangxu@stanford.edu}.
This research is supported by the NSF Grants IIS-1838124, CCF-1850743, and CCF-1856424.
}}

\begin{document}

\maketitle

\begin{abstract}

Convex optimization with feedback is a framework where a learner relies on iterative queries and feedback to arrive at the minimizer of a convex function. It has gained considerable popularity thanks to its scalability in large-scale optimization and machine learning. The repeated interactions, however, expose the learner to privacy risks from eavesdropping adversaries that observe the submitted queries. In this paper, we study how to optimally obfuscate the learner's queries in convex optimization with first-order feedback, so that their learned optimal value is provably difficult to estimate for an eavesdropping adversary. We consider two formulations of learner privacy: a Bayesian formulation in which the convex function is drawn randomly, and a minimax formulation in which the function is fixed and the adversary's probability of error is measured with respect to a minimax criterion. 

Suppose that the learner wishes to ensure the adversary cannot estimate accurately with probability greater than $1/L$ for some $L>0$. Our main results show that the query complexity  overhead is additive in $L$ in the minimax formulation, but multiplicative in $L$ in the Bayesian formulation.  Compared to existing learner-private sequential learning models with binary feedback, our results apply to the significantly richer family of general convex functions with full-gradient feedback. Our proofs learn on tools from the theory of Dirichlet processes, as well as a novel strategy designed for measuring information leakage under a full-gradient oracle. 



\end{abstract}

\section{Introduction}

Convex optimization with feedback is a framework in which an learner repeatedly queries an external data source in order to identify the optimal solution of a convex function. This interactive nature of the framework, however, is a double-edged sword. On the one hand,  iterative optimization methods offer inherent scalability since the learner is not required to possess the entire function from the start.  As such, it has found applications in large-scale distributed machine learning systems, such as Federated Learning \cite{mcmahan2017communication, FL_google2017}, where a learner interacts with millions of individual users (data providers) in order to perform training. On the other hand, the repeated interactions with external entities exposes the {learner} to potential adversaries who may steal the learned model by eavesdropping on the queries exchanged during the training process, a woe especially poignant when the system involves a large number of data providers, many of which could be an eavesdropper in disguise (\cite{juuti2019prada}, \cite[Section 4.3]{kairouz2019advances}). 

To formulate the model stealing attacks and quantify the learner's privacy, we adopt the framework of Private Sequential Learning proposed in a recent line of research,
aimed at quantifying the extra query complexities the learner has to suffer in order to ensure the submitted queries provably conceal the learned value \cite{tsitsiklis2018private, xu2018query, xu2019optimal}. The model is centered around a binary search problem where a learner tries to estimate an unknown value $X^* \in [0,1]$ by sequentially submitting queries and receiving binary responses, indicating the position of $X^*$ relative to the queries. Meanwhile, an adversary observes all of the learner's queries but
not responses, and tries to use this information to estimate $X^*$. The learner's goal is to design a querying strategy with a minimal number of queries so that 
she can accurately estimate $X^*$ while ensuring that the eavesdropping adversary cannot reliably estimate $X^*$. Progress has been made towards understanding the optimal querying strategies in this problem,  and upper and lower bounds on the query complexity have been developed that differ by additive constants in the case where the learner's queries are noiseless \cite{tsitsiklis2018private, xu2019optimal}, and 
are order-wise optimal in the case of noisy queries \cite{xu2019optimal}. 

While the original binary search formulation provides valuable insights, its assumption that the learner only has access to binary feedback is a severe restriction when it comes to modeling convex optimization.  Indeed, most  real-world applications  provide the learner access to significantly richer feedback such as a full gradient (e.g., model training in machine learning). We elaborate further on the potential applications of our model in Section~\ref{sec:rema_model_res}.

The main purpose of the present paper is to take a step towards closing this gap by studying learner-private optimization with general convex functions and a full-gradient oracle.  In a nutshell, our results demonstrate that the  most prominent features of the query complexity in the binary search model extend gracefully to the general convex optimization setting. However, to establish that this is the case is far from trivial. A major difficulty stems from the significantly enriched functional class: unlike in a binary search problem where the ground truth is fully described by a scalar (location of $X^*$), we will see that the private query complexity crucially depends on the shapes of the convex functions in a family, and not just the locations of their minimizers. 


This added richness necessitates the development of both new problem formulations and analytical techniques. We propose in this paper two new learner-privacy frameworks: a new minimax formulation, as well as a Bayesian formulation that generalizes earlier Bayesian private sequential learning to a full-gradient oracle. A number of new techniques are developed to analyze query complexity under these formulations: we introduce tools from the theory of Dirichlet processes to construct priors that convey the richness of the model. Tools from nonparametric Bayes theory are deployed for the analysis under such prior distributions. In addition to an enriched functional class, another fundamental challenge lies in the richness of the feedback. Unlike the binary search model, the responses aligns with the location of the query and the shape of the unobserved convex function to a great extent.
In the face of a more powerful learner equipped with a full-gradient oracle, we rely on a more sophisticated line of analysis to gauge the amount of information the responses reveal. We will discuss in more detail these ramifications in Section \ref{sec:rema_model_res}.

\subsection{Related Work}

\emph{Private information retrieval (PIR) and private function retrieval (PFR)} Our model formulation bears some similarities with the PIR~\cite{abadi1989hiding, chor1995private, gasarch2004survey} and PFR~\cite{mirmohseni2018private} framework. However, there are major distinctions which result in completely different dynamics between the learner and the adversary. In PIR, the database is assumed to contain a vector $(x_i)_{i\leq N}$. The learner's goal is to learn the evaluation $x_i$ at some index $i$ by querying the database, while preventing the database (adversary) from learning the value of $i$. The PFR problem is formulated similarly, except that the database is indexed by functions. Note that in PIR/PFC, the private index is assumed to be known to the learner a priori. In contrast, in our framework, the private information $X^*$ is something the learner herself is in the process of discovering. As a result, our problem is posed as a sequential learning problem. It has natural applications in model stealing attack prevention, where eavesdropping adversaries attempt to steal the model parameters by participating in the model training process. The fundamental difference between the two settings also leads to completely different techniques for analysis. For us, privacy is ensured by utilizing the adversary's lack of knowledge on the responses, which is not the case in PIR/PFC.

\emph{Data-owner privacy models} Similar to Private Sequential Learning, the private convex optimization problem we consider diverges significantly from the existing literature on differentially private iterative learning~\cite{song2013stochastic,abadi2016deep,agarwal2018cpsgd,jain2012differentially,melis2019exploiting}, a key difference being that the latter focuses on protecting data owners' privacy rather than learner's privacy. To protect data owners' privacy, the notion of differential privacy~\cite{Dwork08} is often adopted and privacy is often achieved by injecting calibrated noise at each iteration of the learning algorithms. In contrast, our work focuses on preventing the adversary inferring the learned model, which is conceptually closer to recent studies of information-theoretically sound obfuscation in sequential decision-making problems~\cite{fanti2015spy, luo2016infection, tsitsiklis2018delay, erturk2019dynamically, tang2020privacy}.
See~\cite{xu2019optimal} for a comprehensive discussion on the distinction between data-owner privacy models and this line of work. 

\emph{Strategic learning} In aiming to prevent modeling stealing, our work aligns with a growing literature on strategic learning and prediction \cite{aridor2020competing, ben2017best, ben2019regression, immorlica2011dueling, mansour2017competing}. These papers consider strategic learners who have gained access to their competitor's predicted samples, or even the competitor's entire predictive model. Then, they artificially adjust their own predictive model in order to outperform those of their competitors. In general, in equilibrium such competition could not only harm utilities for the learners involved, but also lead to lower overall social welfare, as defined by the prediction quality experienced by end consumers. Our work thus helps to preempt such pitfalls by providing a theft-proof framework for models training and adaptive data collection.

\section{The Model: Learner-Private Convex Optimization}
\label{sec:formulation}

We now introduce our model, dubbed Learner-Private Convex Optimization. The emphasis on the learner's privacy here is to distinguish our model from other forms of private sequential learning, especially those that focus on protecting the privacy of data owners (See proceeding discussion in the Introduction). 


\paragraph{Learner} Let $\mathcal{F}$ be a family of $\R$-valued convex functions with domain $[0,1]$, such that all elements in $\calF$ admit a unique minimizer. Suppose there is an unknown \emph{truth} $f^*\in \mathcal{F}$ with the minimizer $X^* := \arg\min_{x} f^*(x)$.  Fix $n \in \N$. Our decision maker is a \emph{learner} who wants to identify $X^*$ by sequentially submitting a total of $n$ queries in $[0,1]$ to an oracle. For the $i$th query,  $q_i$, the oracle returns a response $r_i$ that is equal to the gradient of $f^*$ at $q$: 
\begin{equation}
r_i = (f^*)'(q_i). 
\end{equation}
If $f^*$ is not differentiable at $q_i$, then $r_i$ is an arbitrary subgradient of $f^*$ at $q_i$.  We assume that the learner is allowed to introduce outside randomness, in the form of a random seed $Y$ that takes value in a finite discrete alphabet. Formally, we denote by $\phi$ the \emph{learner's strategy}, which consists of a sequence of mappings $\phi_0,\phi_1,...,\phi_{n-1}$ such that the $i$th query is generated as a function of all previous responses and the random seed: 
\begin{equation}
q_i=\phi_{i-1}(r_1,...,r_{i-1},Y). 
\end{equation}
Once the querying process is terminated, the learner constructs an estimator of the optimizer $X^*$,  $\widehat{X}$, based on the $n$ responses. We say that the learner strategy $\phi$ is $\epsilon$-accurate, if
\begin{equation}
\mathbb{P}_f\left\{\left|\widehat{X}-x \right|\leq \epsilon/2\right\}=1, \quad \forall f \in \calF, 
\end{equation}
where $x$ is the minimize of $f$ and the $\pb_f$ indicates the induced probability law when the truth $f^*$ is equal to $f$, and the probability is measured with respect to the randomness in the random seed, $Y$. 

\paragraph{Adversary} Meanwhile, an adversary is trying to learn $X^*$ by eavesdropping on the learner's queries: we assume that the adversary observes all $n$ queries submitted by the learner, but not their responses. Denote by  $\wt{X}$ the adversary's estimator, which is a (possibly random) function of $(q_i)_{i=1, \ldots, n}$. 
Wary of such an  adversary, the high-level objective of the learner are  to $(1)$ generate a  query sequence that is largely ``uninformative'' towards $X^*$, and  $(2)$ at the same minimizing the number of queries needed, $n$.

We next formalize in what sense a learner's strategy can be private. Generally speaking, a learner strategy is private if we can ensure that the adversary's estimator $\wt X$ is \emph{not accurate}. Importantly, different definitions of the adversary's accuracy will lead to drastically different definitions of privacy, and consequently, distinct algorithms,  guarantees and domains of applications.  In this paper, we will analyze two privacy metrics, Bayesian and minimax, that parallel the two paradigms in the statistics literature. The Bayesian formulation extends the Bayesian private learning model in \cite{tsitsiklis2018private}, while the minimax  formulation is new. 



\paragraph {Minimax} The truth $f^*$ is a deterministic but unknown function in $\mathcal{F}$. 
We say that a learner strategy $\phi$ is $(\delta,L)$-private if 
\begin{equation}
\label{eq:def.privacy.minimax}
\sup_{\widetilde{X}}\inf_{f\in \mathcal{F}}\mathbb{P}_f\left\{\left|\widetilde{X}-x\right|\leq\delta/2\right\}\leq 1/L,
\end{equation}
where the probability is measured with respect to the internal randomness employed by the learner's querying strategy and that used in the adversary's estimator. {In other words, the learner strategy is considered private if the adversary's minimax risk is large.}

\paragraph{Bayesian} The truth $f^*$ is drawn from a prior distribution $\pi$, a probability distribution over $\calF$.  We say that a learner strategy $\phi$ is $(\delta,L)$-private if
\begin{equation}
\sup_{\widetilde{X}} \pb \left\{\left| \widetilde{X} - X^*\right|   \leq \delta/2\right\} \leq 1/L,
\end{equation}
where the probability is measured with respect to all randomness in the system, including the prior $\pi$ and  any internal randomness employed by the learner's querying strategy and the adversary's estimator. 

\paragraph{Private query complexity}  Finally, we have come to the main metric of interest. In both the minimax and the Bayesian formulations, we define the optimal query complexity, $N(\epsilon,\delta,L)$, as the least number of queries necessary for there to exist an $\epsilon$-accurate learner strategy that is also $(\delta, L)$-private: 
\begin{align*}
N(\epsilon,\delta,L)=& \min\{n:\exists \phi\text{ with at most }n\text{ queries, }\\
& \text{that is }\epsilon\text{-accurate and }(\delta,L)\text{-private}\}.
\end{align*}

\section{Main Results}

\subsection{Minimax formulation} We will assume that the function class $\calF$ satisfies the following assumption: 

\begin{assumption}[Complexity of $\mathcal{F}$]
\label{assump:complexity}
Fix $f\in \mathcal{F}$ and interval $I\subset[0,1]$ that contains the minimizer of $f$.  Then, for every $x \in I$, there exists $g \in \mathcal{F}$ such that $g$ is minimized at $x$, and the gradient of $f$ and $g$ coincide outside of $I$.
\end{assumption}
Assumption~\ref{assump:complexity} is needed to rule out trivial cases where a learner may exactly pinpoint the location of the minimizer solely by looking at far-away gradients. We show in Section \ref{sec:proof}
that this richness assumption on $\mathcal{F}$ is in some sense necessary. 
Examples of function classes that satisfy Assumption~\ref{assump:complexity} include the set of all convex functions on $[0,1]$, and the set of all piecewise-linear convex functions on $[0,1]$. The next theorem is our main result for the minimax formulation: 

\begin{theorem}[Minimax Query Complexity]
\label{thm:minimax}
Assume that $\mathcal{F}$ satisfies Assumption~\ref{assump:complexity}. If $2\epsilon\leq\delta\leq 1/L$, then\footnote{Here and subsequently $\log$ refers to logarithm with base $2$.}
\[
2L+\log\frac{\delta}{\epsilon}-2
\leq N(\epsilon,\delta,L) \leq  
\begin{cases}
2L+\log\frac{\delta}{\epsilon} & \text{ if } L \ge \log \frac{1}{\delta} \\
 L+ \log\frac{1}{\epsilon} & \text{ o.w.} 
 \end{cases}
 \, .
\]
\end{theorem}

Note that if there were no privacy consideration, the minimax optimal query complexity would be $\log (1/\epsilon)$.
Thus under the minimax formulation, a higher level of privacy $L$ leads to an 
\emph{additive} overhead in the  optimal query complexity, that is at most about $2L$.

\begin{remark}[Multidimensional Extensions]
By considering a separable class of functions, and using the $\ell_\infty$ norm to measure the error of the learner and the adversary's estimators, Theorem~\ref{thm:minimax} can be extended to $d$ dimensions. The upper and lower bounds of the query complexity take the same form, with $L$ replaced with $L^{1/d}$. See the supplementary material for the precise statement and proof.
\end{remark}

\begin{figure}[ht]
\centering
\begin{subfigure}{.5\textwidth}
  \centering
  \includegraphics[width=\linewidth]{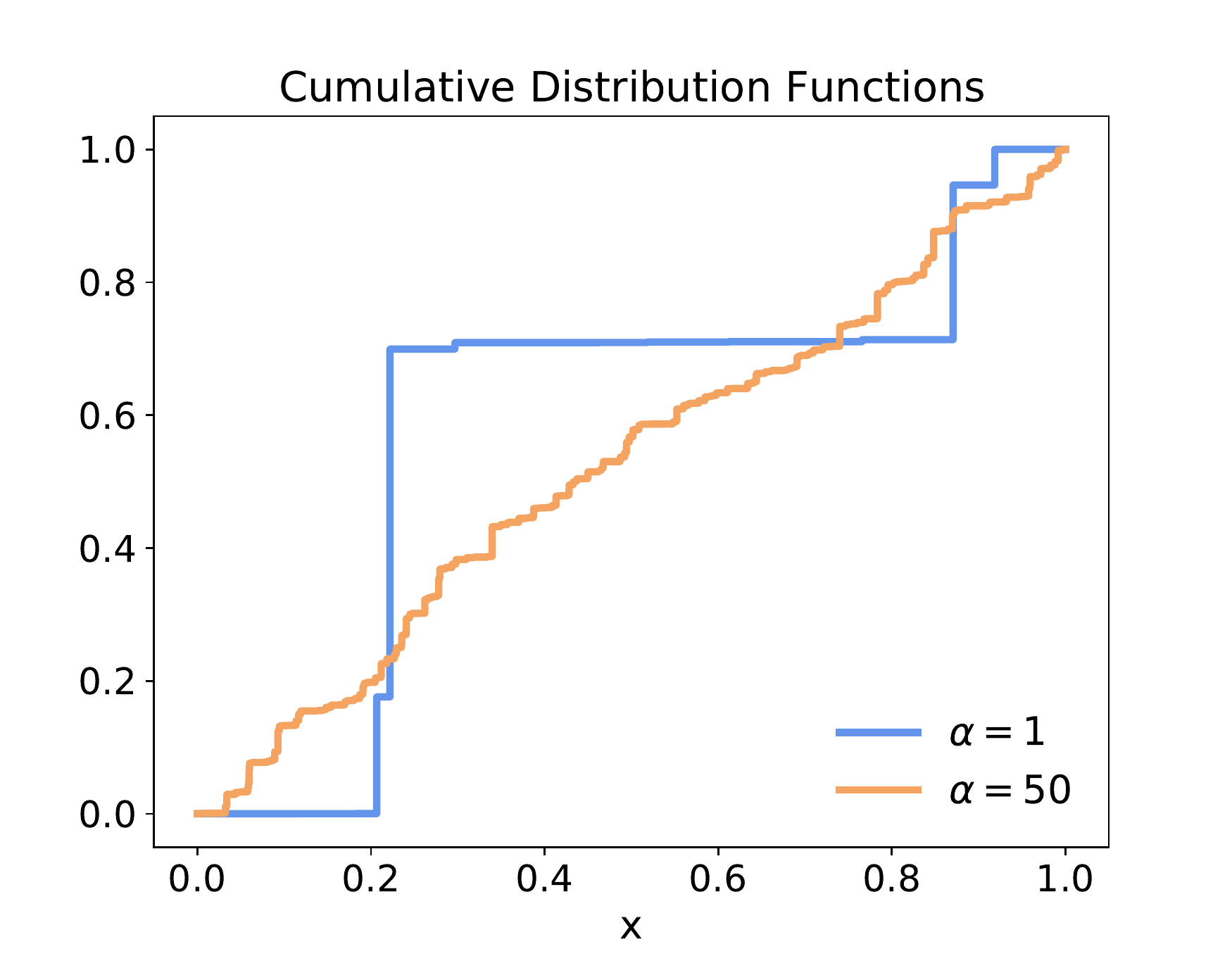}
  \label{fig:sub1}
\end{subfigure}%
\begin{subfigure}{.5\textwidth}
  \centering
  \includegraphics[width=\linewidth]{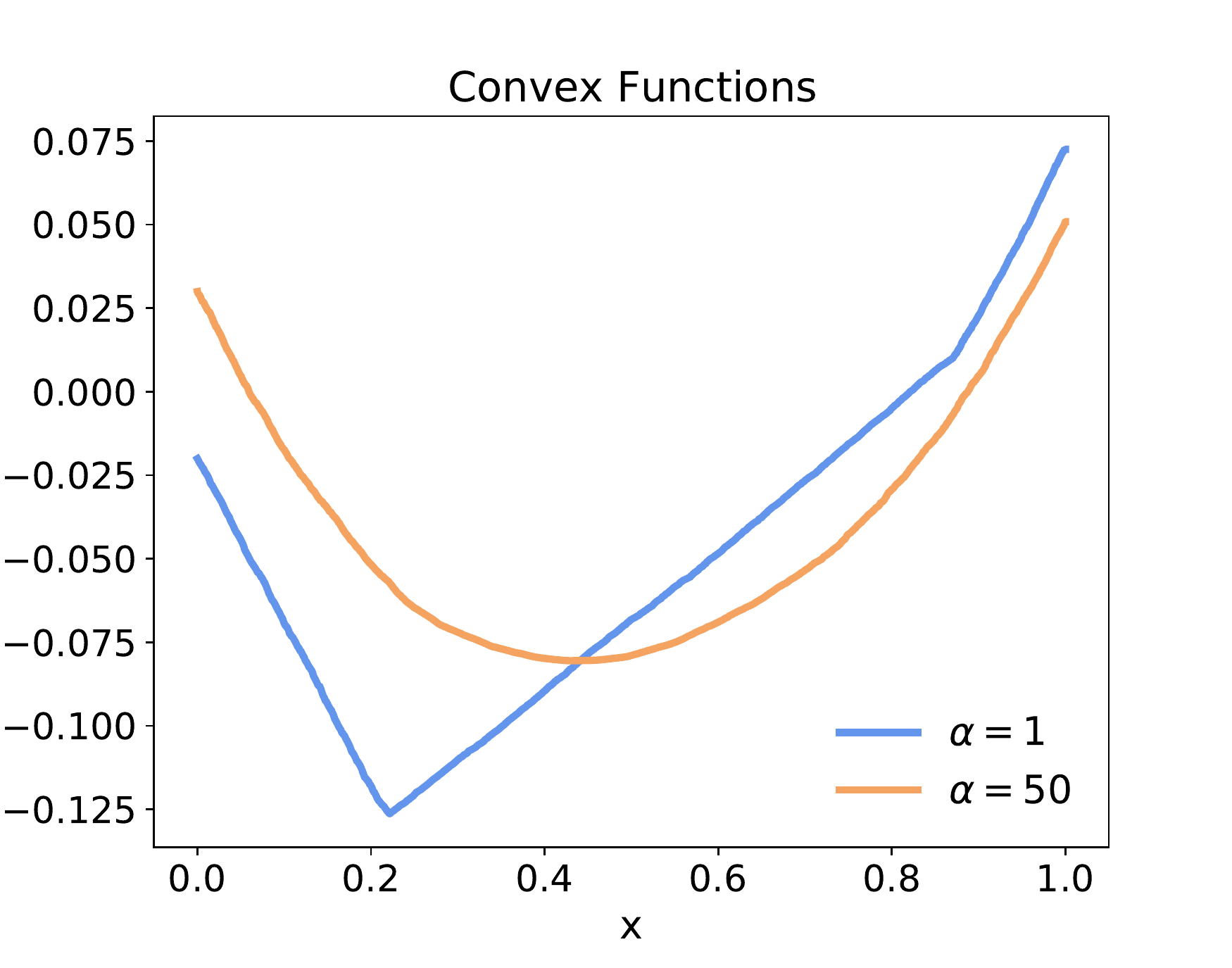}
  \label{fig:sub2}
\end{subfigure}
\caption{The left figure exemplifies realizations of $F$ following the Dirichlet Process with base function $\lambda_{[0,1]}$ and different concentration parameters $\alpha$. The right figure shows the corresponding convex functions $f^*$, with $\gamma_+=0.5$ and $\gamma_-=-0.5$.
}
\label{fig:Dir}
\end{figure}

\subsection{Bayesian formulation}

In the Bayesian formulation, we seek a function class and prior distribution that are sufficiently rich to capture real-world data, while at the same time amenable to analysis. A good candidate in this respect is the so-called Dirichlet process, a family of measures over non-decreasing functions,  which we will use to model the gradient function of $f^*$. Dirichlet processes are fundamental objects in nonparametric Bayes theory and widely used in Bayesian isotonic regression for modeling monotone functions \cite{lavine1995nonparametric, bornkamp2009bayesian, neelon2004bayesian}. We begin by defining a Dirichlet process: 

\begin{definition}[Dirichlet Process]
Given a base probability measure $\mu_0$ on $\mathcal{X}$ 
 and a concentration parameter $\alpha>0$. A random probability measure $\mu$ over $\mathcal{X}$ is said to follow the Dirichlet process $\text{DP}(\mu_0,\alpha)$, if for any finite partition of $\mathcal{X}=\cup_{i\leq n}\mathcal{X}_i$, 
\[
(\mu(\mathcal{X}_1),...,\mu(\mathcal{X}_n))\sim \Dir((\alpha \mu_0(\mathcal{X}_1),..., \alpha\mu_0(\mathcal{X}_n))),
\]
where $\Dir(c)$ denotes the Dirichlet distribution over the $n$-dimensional simplex $\Delta^{n-1}$ with density
\begin{equation}
g_{\Dir(c)}(x_1, \ldots, x_n) \propto \prod_{i=1}^n x_i^{c_i-1}, \quad x \in \Delta^{n-1}.
\end{equation}
\end{definition}
We now construct the prior distribution of $f^*$ using a Dirichlet process. The prior is parameterized by two quantities: 
\begin{enumerate}
\item a concentration parameter $\alpha>0$, which controls the dispersion of the distribution of the minimizer;
\item a probability distribution $\eta$ over $[0,1]$, which captures the range of gradients of $f^*$. We assume that $\eta$ admits a density that is bounded from above and away from $0$ (e.g., $\text{Unif}[0,1]$). 
\end{enumerate}

\begin{definition}[Bayesian Prior using Dirichlet Process] Fix $\alpha$ and $\eta$. Denote by $\lambda_{[0,1]}$ the Lebesgue measure restricted to $[0,1]$. Then, the prior $\pi$ corresponds to the following procedure for generating $f^*$: \footnote{Note that in this definition we have restricted the gradients to lie in $[-1,1]$ and the function $f^*$ to have zero intercept. Both restrictions are without loss of generality, since any constant offset will not change the location of a minimizer and similarly our results will carry through if one wishes to incorporate a different gradient scaling factor. }
\begin{enumerate}
\item Sample $\gamma_+$ from $\eta$. Set $\gamma_-=-\gamma_+$.


\item Sample $\mu$ from the Dirichlet process with concentration parameter $\alpha$ and base distribution $\lambda_{[0,1]}$. Let $F$ be the cumulative distribution function of $\mu$.
\item Set $f^*(x) =\gamma_- x +  \int_0^x (\gamma_+-\gamma_-)F(t)dt$, for $x\in [0,1]$.
\end{enumerate}
\end{definition}

{Note that $\left(f^*(x)\right)'=\gamma_+\left(2 F(x) -1\right)$ and thus the minimizer $X^*$ of $f^*$ corresponds to the median of $F$, or
more precisely the smallest $x$ for which $F(x) \ge 1/2$.}
By construction, $F$ is a monotone simple function that consists of countably many points of discontinuity that are dense on $[0,1]$. Its level of discreteness is modeled through the concentration parameter $\alpha$. For a small $\alpha$, the increase of $F$ from $0$ to $1$ is mostly from a few abrupt jumps, and the convex function $f^*$ resembles a piece-wise linear function with finitely many pieces; as $\alpha$ grows, the increase of $F$ becomes more gradual, and $f^*$ starts to concentrate around a smooth quadratic function. See Figure~\ref{fig:Dir} for some realizations of the distribution function $F$ and the corresponding convex function $f^*$ for different value of $\alpha$.\footnote{To plot the convex functions together, we shift them by some constants on the $y$-axis. This shift is irrelevant to the optimization task since the response only contains gradient information.}

The following theorem is our main result for the Bayesian formulation.

\begin{theorem}[Bayesian Query Complexity]
\label{thm:bayes}
Fix $\alpha>0$. Suppose that $2\epsilon\leq \delta< \frac{1}{2LH_\alpha}$, with $H_\alpha=(3+2e^{-1})\alpha+14$. Then
\[
c_1 L\log\frac{\delta}{\epsilon}
\leq N(\epsilon,\delta,L)
\leq L\log\frac{\delta}{\epsilon}+c_2 L+\log\frac{1}{\delta L},
\]
where $c_1$,$c_2$ are positive constants that only depend on $\alpha$ such that $c_1\rightarrow 1$ as $\alpha\rightarrow 0$. 
\end{theorem}

The above theorem shows that, in the Bayesian formulation,  the query complexity overhead due to privacy constraints scales \emph{multiplicatively} with respect to the privacy level $L$. Note that this is substantially higher than the minimax setting where such overhead is only additive in $L$. When $\alpha \to 0,$ $F$ converges to a step function and our query complexity bounds recover the existing ones in
the binary search problem~\cite{xu2018query}, showing that $N(\epsilon,\delta,L) \sim L\log \frac{1}{\epsilon} $ as $\epsilon \to 0$ 
for fixed $\delta, L$.


\section{Discussion}
\label{sec:rema_model_res}

In this section, we examine some real-world applications of our privacy model and discuss some of the most salient features of our main results and modeling assumptions. 

{\bf Motivating examples} A learner naturally suffers from privacy breaches if the learning process involves interactions with third-party users. An example would be the aforementioned Federated Learning framework. A typical Federated Learning model training process can be posed as iterative optimization of some unknown function. Iterations of model updates are generated from the feedback from a large number of users (see e.g. the \emph{FederatedAveraging} algorithm~\cite{mcmahan2017communication}). Since the model updates (queries) are broadcasted to the participating users, the learner is exposed to eavesdropping attacks. Due to the high cost of large-scale model training, it is of great importance to protect the learner from such privacy breaches, and do so at a minimal cost~\cite{kairouz2019advances}.

Another potential application is pricing optimization, where the goal is to learn the optimal release price of a product by conducting market experiments at test price points (queries). See \cite{xu2019optimal,tsitsiklis2018private} for more detailed discussions on the Federated learning and pricing optimization examples.

Given the close connection between convex and monotone functions, our work can also be applied to learning monotone functions, for example to clinical dose-response studies~\cite{ramgopal1993nonparametric, bornkamp2009bayesian}. In dose-response analysis, the potency curve $\mu(x)$ is a monotone function that models the treatment effectiveness as a function of the dosage. An important problem is to estimate the minimum effective dose ($\mathsf{MED}$)
\[
\mathsf{MED}=\min_x\{x:\mu(x)>\mu(0)+\Delta\}
\]
for some threshold $\Delta$. Note that the $\mathsf{MED}$ is the minimizer $X^*$ of some unknown convex function $f^*$ (e.g. $f^*(x)=\int_0^t \mu(t)dt-[\mu(0)+\Delta]x$).
We also remark that the Dirichlet process is widely used in isotonic regression for modeling monotone functions~\cite{lavine1995nonparametric, bornkamp2009bayesian}, as we do when modeling the gradient of the convex function.

\paragraph{Applying the Bayesian and Minimax privacy criteria.}  Our results show that the two privacy criteria lead to distinct query complexity scalings, so it would be instructive to understand in what application domain each metric is most applicable. The Bayesian formulation is more straightforward: both the adversary and the learner are assumed to have access to the historical data that forms the prior distribution, and all probabilities in various guarantees are measured with respect to such shared common knowledge. We expect the Bayesian formulation to be most relevant in data-driven machine learning and online optimization such as in Federated Learning and pricing optimization; the aforementioned dose-response analysis is also a natural application of the Bayesian formulation due to the close connection between potency curves and convex functions.


The minimax formulation is a new metric proposed in this paper, and we discuss here some nuances with this definition. Note that minimax guarantees in traditional statistical learning are typically the strongest, since they hold over any adversary choice of problem parameter. However, in our setting, the minimax formulation provides arguably  the {weakest} privacy guarantee due to the negation inherent in its definition: a learner strategy is minimax-private as long as there does \emph{not} exist a minimax-accurate adversary estimator. For instance, even if adversary is able to accurately predict $X^*$ under the majority of functions in $\calF$, failing only over a small subset, the learner can still proclaim its strategy to be private under the minimax formulation. As a result, we see that the query complexity is significantly lower for the same combination of $(\delta, L)$ under the minimax formulation than the Bayesian one. 

If the minimax  formulation is weak, then when is it an appropriate metric? We tends to believe that the formulation is appropriate if the application requires the adversary to  use minimax-accurate estimators (in the traditional statistical sense).  One interesting example is in law and criminal justice. Here, a prosecutor should have to prove that the accuracy of any conclusion drawn from evidence holds up \emph{regardless} of the value of a certain hidden parameter. Indeed, many legal systems currently require that criminal convictions be reached only if the evidence can prove guilt ``beyond reasonable doubt'' (cf.~\cite{1881miles, 1895coffin, young2001juries}). Any supposed prior on crucial, unobserved parameters can be ill-defined and potentially discriminatory. Other potential applications include autonomous driving~\cite{ren2020improving}, where the performance guarantee of an estimator needs to be valid in the worst case, for the sake of public safety. In these examples, a minimax-private learner strategy will effectively prevent the adversary from coming up with \emph{any} viable estimator, thus render the adversary powerless. 


\paragraph{Comparisons with private sequential learning.} As mentioned in the Introduction, our
convex optimization framework generalizes the Private Sequential Learning (PSL) model. As such, the two settings share similarities (as one would expect), as well as some marked differences. Recall that in the PSL framework, the responses are {binary} and only indicate whether the minimizer is to the left or right of a given query; this is equivalent, in our setting, to returning only  the sign of the gradient. The minimax and Bayesian formulations proposed in this paper parallel the deterministic and Bayesian formulations in PSL, respectively. Our minimax formulation is new, but it does have a fundamental connection to the deterministic formulation of PSL, where a learner strategy is considered private if its queries are guaranteed to generate a large set of ``plausible'' targets (information set), with a large covering number; we  explore this formally in Section \ref{sec:proof.minimax}.
 Our Bayesian formulation is a natural generalization of the Bayesian PSL model: we now assign a prior over the entire function, as opposed to only the location of the minimizer. Notably, our Bayesian formulation recovers the original Bayesian PSL problem in the limit where the concentration parameter $\alpha$ in the Dirichlet prior approaches $0$. As such, our Bayesian formulation includes the original Bayesian PSL model as a special case.  

Our main results recover similar dependencies on the level of privacy, with overheads that are additive and multiplicative in $L$ in the minimax and Bayesian formulations, respectively. The upshot in our setting is that  the results are established in a substantially more general setting of convex optimization. 


There are several major differences that distinguish our private convex optimization framework from the PSL model. First and foremost, the learner now has access to the entire gradient instead of only its sign. A most direct implication of this enriched information structure is that, when analyzing the amount of information leakage of a  learner strategy, we will have to keep track of the distributions over target {functions}, as opposed to only the minimizers, as was the case in PSL. Moreover, when the learner has access to full gradients, it is in principle possible for the learner to gather information about the minimizer's precise location even from queries that are submitted far away from the minimizer, which was not possible within bisection search. For instance, if the underlying target function is known to be quadratic, then two queries placed {anywhere} are sufficient to uncover the minimizer. To address these complexities, our goal is to precisely measure the amount of information about the minimizer that the learner {and} adversary may obtain from a given sequence of queries. We will do so both by developing more sophisticated information theoretic arguments, and by exploiting structural properties of the Dirichlet process. 



\paragraph{Open questions} Our results leave open a number of questions. For the Bayesian query complexity in one dimension, there remains a gap between the leading constants in the upper and lower bounds, in the regime where $\alpha$ is bounded away from zero. Generalizing the main theorems to a multi-dimensional setting, where $x \in \R^d, d\geq 2$, is also interesting and practically relevant.  We take a first step in this direction by extending our results to multi-dimensional separable functions (see supplementary material), while the general case with non-separable objective functions remains open and appears to be challenging. Our problem formulation only considers first-order feedback. An interesting direction is to consider convex optimization with more general types of feedback, e.g., bandit feedback~\cite{agarwal2013stochastic}.

\paragraph{A different notion of minimax privacy in \cite{tang2020optimal}}
A recent work \cite{tang2020optimal}  also aims to extend the private sequential learning model of \cite{tsitsiklis2018private} to convex optimization. They use a different notion of minimax privacy criteria that bear some superficial similarities to ours. 
However, the definition of privacy in  \cite{tang2020optimal} contains crucial errors that render it vacuous,  in the sense that 
 there cannot exist any private learner strategy satisfying that definition. 
To be precise, here is Definition 2 of \cite{tang2020optimal}: fix $\epsilon, \delta\in (0,1)$. A learner strategy is said to be $(\epsilon, \delta)$-private if for any adversary estimator $\wt X$ \emph{and} any truth $f \in \calF$, 
\begin{equation}
\pb_f( \mbox{err}(\wt X, f) \leq \epsilon )\leq \delta, 
\label{eq:tang_def}
\end{equation}
where $\mbox{err}(\cdot, \cdot)$ is a certain error function which measures the discrepancy between the adversary estimator and the true minimizer. For instance, in our example $\mbox{err}(\wt X, f) = |\wt X - \arg\min f(x)|$. 

The problem with this privacy definition is  that it can never be satisfied by any learner strategy. Indeed, for any $f\in \calF$ with minimizer $x^*$, there always exists an adversary estimator that trivially yields zero estimation error with probability one: simply set $\wt X = x^*$, without even taking into account the queries. Under this trivial estimator, we automatically have $\pb_f( \mbox{err}(\wt X, f) =0)= 1$, so \eqref{eq:tang_def} cannot possibly hold uniformly across all adversary estimators and all $f$. Unfortunately,  this would further suggest that the analysis and conclusions in \cite{tang2020optimal} contain errors as well. 


\section{Proof of Main Results}\label{sec:proof}

We present in this section the proofs of our main results. We begin by giving an overview of the key steps. 

\subsection{Overview of Main Ideas}\label{sec:proof.overview}



\paragraph{Minimax setting} 

Since the response contains the full gradient information, the key challenge in the analysis is to track the amount of information available to the learner. Note that aside from the directional information $\mathds{1}\{X^*\geq q\}$, the response for a query $q$ contains additional information on $(f^*)'(q)$. The key message in the proof under the minimax setting, is that under the Assumption~\ref{assump:complexity} on the richness of the family of functions, only the directional information is relevant to the learning task. Therefore, it suffices to only track the learner's knowledge with the directional information from the responses.

Starting with the upper bound, we design a querying strategy that is $\epsilon$-accurate, $(\delta,L)$-private, and submits at most $\max\{2L+\log(\delta/\epsilon), L+\log(1/\epsilon)\}$ queries. In particular, our querying strategy only utilizes the directional information of the gradient responses. Firstly, note that since the gradient responses contain the binary directional information, the learner can always check whether an interval contains $X^*$ by querying the two endpoints. We refer to a pair of queries at $q$ and $q+\epsilon$ as a {\it guess}. The key privacy-ensuring mechanism is to check $L$ guesses that are $\delta$ apart from each other. By doing so, the learner manually plants $L$ possible locations for $X^*$ that an adversary cannot rule out without observing the responses, thus achieving $(\delta,L)$-privacy.

To prove the lower bound, we need to show that a querying strategy that only utilizes the directional information can be optimal. Firstly, let us give a heuristic argument of why only the gradient information is relevant to learning $X^*$ under Assumption~\ref{assump:complexity}.
Given $(f^*)'(a)<0$ and $(f^*)'(b)>0$, under Assumption~\ref{assump:complexity}, $X^*$ can be anywhere between $a$ and $b$ regardless of the value of the gradients $(f^*)'(a)$, $(f^*)'(b)$. We should point out that the richness assumption is necessary. For example suppose $\mathcal{F}$ is the family of convex polynomial functions with fixed degree $d$. Then the learner can solve for the $X^*$ by submitting $d$ distinct queries at arbitrary locations, making both learning and obfuscation trivial.

The lower bound proof contains two main ingredients.
\begin{enumerate}[(a)]
\item Step 1: Rigorously justify the claim that under Assumption~\ref{assump:complexity}, the learner does not benefit from the additional gradient information aside from the one-bit directional response. In particular, we show that the learner cannot search faster than the bisection method on any interval $I\subset [0,1]$. Therefore, for each interval of length $\delta$, it takes at least $\log(\delta/\epsilon)$ queries in $I$ to achieve $\epsilon$-accuracy, in the worst case.

\item Step 2: Relate the adversary's statistical performance to the size of the information set~\cite{tsitsiklis2018private} of a query sequence $q$, defined as
\[
\mathcal{I}(q)=\left\{x\in [0,1]: \exists f\in \mathcal{F}\text{ and } y,\;\;s.t.\;\; x=\arg\min f,\text{ and }q(f,y)=q\right\}.
\]
The information set contains all possible values of $X^*$ that could lead to the query sequence $q$. We show that to ensure the adversary achieves $\delta$-accuracy with probability at most $1/L$,
there must be some $q$ for which the $\delta$-covering number of $\mathcal{I}(q)$ is at least $L$. Note that from the $\epsilon$-accuracy requirement, each member of $\mathcal{I}(q)$ is sandwiched between a pair of queries in $q$ that are at most $\epsilon$-apart. 
Therefore, $q$ contains at least $L$ such pairs of queries, contributing a total of $2L$ queries. 
\end{enumerate}

After performing these two steps, some challenges remain. The functions associated with $q$ (in step 2) may not coincide with the worst-case instances that arise from step 1.
Therefore, the remaining task is to combine the two lower bounds $\log(\delta/\epsilon)$ and $2L$. For this step, we show the existence of some interval $I$, such that for some $f$ minimized in $I$, the learner must pay not only the $\log(\delta/\epsilon)$ queries for accuracy, but also the $2L$ queries for privacy. The high-level idea behind the proof is to divide $q$ into two sub-sequences $q_\mathsf{before}$, $q_\mathsf{after}$, before and after the $2L$ queries (in step 2) are submitted.
The key observation is that $q_\mathsf{before}$ is shared by a large class of functions whose minimizers lie in some $\delta$-length interval $I$. For all these functions, the cost of $2L$ queries would have been committed in $q_\mathsf{before}$. For at least one of them, an extra cost of $\log(\delta/\epsilon)$ queries must be paid in $q_\mathsf{after}$.

\paragraph{Bayesian setting}
Similar to the minimax setting, the upper bound here is also established by analyzing a constructive algorithm. The key challenge in designing a private learning algorithm in the Bayesian setting arises from the fact that the prior distribution on $X^*$ is always non-uniform under the Dirichlet process model.  In particular, we can no longer simply apply the replicated search strategy from \cite{xu2019optimal}, since the non-uniform distribution of $X^*$ provides the adversary with additional prior information.

To address this difficulty, our key algorithmic idea is to find $L$ intervals that occupy the same prior mass, while at the same time are at least $\delta$-separated from each other. One of these intervals contains the true value $X^*$. On each of the other $L-1$ intervals, we sample a proxy for $X^*$ according to the conditional distribution of $X^*$ restricted to the interval. Via a genie-aided reduction argument, we show that the adversary cannot perform better than a random guess among the $L$ candidates: the truth and the $L-1$ proxies. By construction of the intervals, the $L$ candidates are at least $\delta$ apart. Therefore the adversary cannot achieve an additive error of $\delta/2$ with probability higher than $1/L$.

For the lower bound, the challenge again lies in tracking and quantifying the amount of information the learner gains from the responses. Compared to the binary search model, the full gradient responses can potentially reveal too much information to the learner. To tackle this challenge, our key proof strategy is to find a event on which the learner cannot gather information on $X^*$ too rapidly. The proof follows the following main steps. 

\begin{enumerate}[(a)]
\item Step 1: quantify the learner's information. We adopt the notion of ``learner's intervals", $I_0,I_1,...$. Here, $I_0=[0,1]$ and $I_i$ is the smallest interval that the learner knows to contain $X^*$ after the first $i$ queries. 

\item Step 2: analyze the conditional distribution of $X^*$ over the learner's interval. This is the key step of the proof. We want to find a ``good" event on which the learner does not possess too much information on the location of $X^*$. In this step, we construct an event $\mathcal{B}$, such that
\begin{equation}\label{eq:uniform}
X^*\mid \mathcal{B} \sim \mathrm{Unif}[I_i\cap J],
\end{equation}
where $J$ is an arbitrary subinterval of $[0,1]$. Here, $\mathcal{B}$ is an event that encodes all the information available to the learner up to time $i$, the assumption that $X^*\in J$, and some other desirable feature of the unknown convex function $f^*$. The construction of $\mathcal{B}$ crucially utilizes the stick-breaking characterization of the Dirichlet Process, and the proof of~\eqref{eq:uniform} heavily relies on the self-similarity property. The self-similarity property helps control the amount of information about the shape of $f^*$ inside the learner's interval, given all the queries and responses outside.

\item Step 3: control the speed at which the learner's interval shrinks. Divide $[0,1]$ into $2/\delta$ subintervals $J_1,...,J_{2/\delta}$ of length $\delta/2$, and let $J^*$ denote the subinterval of contains $X^*$. In this step, from~\eqref{eq:uniform}, by integrating over instances of $\mathcal{B}$, and letting $J$ range over the $2/\delta$ subintervals, we show that for some suitably-defined event $\mathcal{A}$,
\begin{equation}\label{eq:learner.interval.big}
\mathbb{E}\left(\log\frac{|I_{i+1}\cap J^*|}{|I_i\cap J^*|} \, \Big\rvert \, \mathcal{A}\right) \geq -\mathbb{P}\left\{q_{i+1}\in I_i\cap J^*\mid \mathcal{A}\right\}\geq -\mathbb{P}\left\{q_{i+1}\in J^*\mid \mathcal{A}\right\}.
\end{equation}

\item Step 4: from~\eqref{eq:learner.interval.big}, via a simple telescoping sum and an application of Jensen's inequality, we can deduce that
\[
\mathbb{E}\left(\text{number of queries in }J^*\mid \mathcal{A}\right) \geq \log\frac{\delta}{2}-\log \mathbb{P}\left(|I_n\cap J^*|\mid \mathcal{A}\right) \geq \log\frac{\delta}{\epsilon},
\]
where the second inequality follows from the $\epsilon$-accuracy requirement.
By consider an adversary who adopts the {\it proportional-sampling} strategy~\cite{xu2018query}, we have for any querying strategy that is $(\delta,L)$-private,
\[
n=\text{total number of queries}\geq L\mathbb{E}\left[\text{total number of queries in }J^*\right] \geq \mathbb{P}\left(\mathcal{A}\right)\cdot L\log\frac{\delta}{\epsilon}.
\]
\end{enumerate}

\subsection{Proof under the Minimax Setting}\label{sec:proof.minimax}

\begin{proof}[Proof of the upper bound in Theorem~\ref{thm:minimax}.]

Define a guess at $q$ as a pair of queries placed at $q$ and $q+\epsilon$. The guess allows the learner to test whether $X^*$ is contained in the $\epsilon$-length interval $[q,q+\epsilon]$. To ensure privacy, we create $L$ potential locations for $X^*$ that are at least $\delta$-separated but induce the same querying sequence. That is achieved by submitting $L$ {\it guesses} that are $\delta$-separated. Once guessed correctly, the learner's accuracy requirement is automatically fulfilled and the remaining queries can be used to conceal $X^*$ from the adversary. 
We consider the cases $\delta\leq 2^{-L}$ and $\delta>2^{-L}$ separately. The querying strategy is contained in Algorithm~\ref{alg:minimax}.


\begin{algorithm}[H]
\caption{Querying Strategy under the Minimax Setting}
\label{alg:minimax}
\begin{algorithmic}[1]
\STATE
Let $I=[0,1]$.
\IF {$\delta\leq 2^{-L}$}
\STATE 
Submit the first guess at $1/2$.

\STATE
Recursively submit the remaining $L-1$ guesses via bisection: if none of the submitted guesses is correct, update $I=[a,b]$ according the gradient $(f^*)'(q)$ at the previous guess $q$. If $(f^*)'(q)\leq 0$, then $X^*\geq q$, so we let the updated $I$ be $[q,b]$; otherwise update $I$ to be $[a,q]$. Submit the next guess at the midpoint of the updated $I$. 

\STATE 
Once a guess is found to be correct, always (do this also for all the remaining guesses) update $I$ to be the right half of $I$, and submit the next guess at the midpoint of the updated $I$. 

\ELSE
\STATE 
Submit the first guess at $0$.
\STATE
Let $K$ be an integer solution in $\{0,1,...,L-1\}$ such that $\ell_K:=2^{-K}/(L-K)\in [\delta,2\delta]$. When $\delta>2^{-L}$, a solution always exists.

\STATE 
Submit the next $K$ guesses via bisection. Update $I$ accordingly. As in the $\delta\leq 2^{-L}$ case, once any guess is found to be correct, always update $I$ to its right half.

\STATE
Divide $I$ into $L-K$ equal length subintervals. Submit the next $L-K-1$ queries at the endpoints of the subintervals (excluding the $2$ endpoints of $I$).

\ENDIF
\IF {none of the guesses is correct}
\STATE
Run bisection search on $I$ until reaching $\epsilon$-accuracy. 
\ELSE
\STATE
Fill the remaining query sequence with trivial queries at $1$.
\ENDIF
\end{algorithmic}
\end{algorithm}


%

We first prove the upper bound in the case $\delta\leq 2^{-L}$. In total, $L+\log(1/\epsilon)$ queries are submitted under Algorithm~\ref{alg:minimax}. The strategy is clearly $\epsilon$-accurate. To see that it is also $(\delta,L)$-private, note that all $f^*$ whose minimizer lies in one of the $L$ intervals $[1/2,1/2+\epsilon]$, $[3/4,3/4+\epsilon]$, ..., $[1-2^{-L}, 1-2^{-L}+\epsilon]$ share exactly the same query sequence. Under Assumption~\ref{assump:complexity}, for each $i$ there exists at least one function $f_i$ minimized at some $x_i\in [1-2^{-i}, 1-2^{-i}+\epsilon]$. When $\delta\leq 2^{-L}$, the $x_i$'s are at least $\delta$ apart from each other. Therefore no adversary can achieve $\inf_{f\in \{f_1,...,f_L\}}\mathbb{P}_f\{|\widetilde{X}-x|\leq \delta/2\}>1/L$.

When $\delta>2^{-L}$, the total number of queries is at most $\log(\delta/\epsilon)+2L+1$. Note that the first guess at $0$ always contains a trivial query at $0$. Removing the trivial query yields a query complexity of $\log(\delta/\epsilon)+2L$. To prove $(\delta,L)$-privacy, note that for if $f^*$ is minimized in one of the $L$ intervals $[0,\epsilon]$, $[1-2^{-i}, 1-2^{-i}+\epsilon]$ for $i\leq K$, or $[1-2^{-K}+i\ell_K, 1-2^{-K}+i\ell_K+\epsilon]$ for $i\leq L-K-1$, 
then they induce the same query sequence. This completes the proof of the upper bound.

\end{proof}

We now turn to the lower bound. As a first step, we prove that if $\mathcal{F}$ satisfies Assumption~\ref{assump:complexity}, then the learner cannot search faster than the bisection method on any interval $I\subset[0,1]$. The lemma below contains a formal statement of this claim. Note that by taking $I=[0,1]$, Lemma~\ref{lmm:complexity} immediately implies a lower bound of $\log(1/\epsilon)$ on the optimal query complexity. 


\begin{lemma}
\label{lmm:complexity}
Suppose $\mathcal{F}$ satisfies Assumption~\ref{assump:complexity}. Let $\phi$ be an $\epsilon$-accurate querying strategy. Then for each $f\in\mathcal{F}$, each interval $I\subset[0,1]$ that contains the minimizer of $f$, and each realization of the random seed y, there exists $\widetilde{f}\in \mathcal{F}$, such that
\begin{enumerate}[(1)]
\item under $\phi$, the query sequence $q(\widetilde{f},y)$ contains at least $\log (|I|/\epsilon)$ queries in $I$;
\item the gradient of $\widetilde{f}$ and $f$ coincide outside of $I$.
\end{enumerate}
\end{lemma}

Next, we prove the lower bound in Theorem~\ref{thm:minimax} assuming correctness of Lemma~\ref{lmm:complexity}. The proof of Lemma~\ref{lmm:complexity} is deferred to the end of this subsection.

\begin{proof}[Proof of the lower bound in Theorem~\ref{thm:minimax}]

A key step in this proof is to connect definition of $(\delta,L)$-privacy with the covering numbers of the information sets. We claim that for a strategy to be $(\delta,L)$-private in the minimax sense, there must be one information set with a large covering number.



Let $\phi$ be a querying strategy that is both $\epsilon$-accurate and $(\delta,L)$-private. Define the information set of a query sequence $q$ as 
\[
\mathcal{I}(q)=\left\{x\in [0,1]: \exists f\in \mathcal{F}\text{ and } y,\;\;s.t.\;\; x=\arg\min f,\text{ and }q(f,y)=q\right\}.
\]
Denote the $\delta/2$-covering number of $\mathcal{I}(q)$ as $N_c(\mathcal{I}(q),\delta/2)$. 
Fix the adversary's strategy to be one that samples uniformly from a $\delta$-covering set of $\mathcal{I}(q)$. 
Since $\phi$ is $(\delta,L)$-private, 
there must exist some $f$ minimized at $x$, for which
\[
1/L> \mathbb{P}_f\left\{\left|\widetilde{X}-x\right|\leq \delta/2\right\} = \mathbb{E}\left[\mathbb{P}_{f}\left\{\left|\widetilde{X}-x\right|\leq \delta/2 \;\Big\rvert\; q\right\}\right],
\]
where the first integration is over $q$ and the second is over the randomness from the adversary's estimation scheme conditional on $q$. Since $x$ is in $\mathcal{I}(q)$, it must be $\delta/2$-close to at least one of the points in the covering set. Therefore for all $q$,
\[
\mathbb{P}_{f}\left\{\left|\widetilde{X}-x\right|\leq \delta/2\;\Big\rvert\; q\right\}\geq \frac{1}{N_c(\mathcal{I}(q), \delta/2)}.
\]
Taking expected value over $q$ on both sides, we have $\mathbb{E}(1/N_c(\mathcal{I}(q),\delta/2))<1/L$. Hence there must exist some query sequence $\bar{q}$ for which $N_c(\mathcal{I}(\bar{q}),\delta/2)>L$. As a result, $\mathcal{I}(\bar{q})$ contains $L$ points $x_1,...,x_L$ that are at least $\delta/2$-apart. 

By definition of  the information set, there exist $f_1,...,f_L\in\mathcal{F}$ and $y_1,...,y_L\in[0,1]$, such that $f_i$ is minimized at $x_i$, and $q(f_i,y_i)=\bar{q}$ for all $i$. 
Notice that for each $i$, $\bar{q}$ must contain a pair of queries at most $\epsilon$-apart that sandwiches $x_i$. Otherwise suppose the closest pair of queries in $\bar{q}$ that contains $x_i$ forms an interval $I$ of size larger than $\epsilon$. Under Assumption~\ref{assump:complexity}, for each $x\in I$, there exists $f\in \mathcal{F}$ for which $f$ is minimized at $x$ and $q(f,y_i)$ is also $\bar{q}$. By taking $x$ to be arbitrarily close to the endpoints of $I$, the $\epsilon$-accuracy requirement is violated since no estimator $\widehat{X}$ can ensure $|\widehat{X}-x|\leq \epsilon/2$ for all $x\in I$. Therefore, the length of $I$ is at most $\epsilon$. Combined with the fact that $x_1,...,x_L$ are $\delta$-separated, and the assumption $\delta\geq 2\epsilon$, we have shown that $\bar{q}$ contains $L$ pairs of distinct queries. Thus the optimal query complexity is lower bounded by $2L$.

To improve the lower bound to the desired $2L+\log(\delta/\epsilon)$, we would like to argue that aside from the $L$ pairs queries in $\bar{q}$, the learner must submit enough queries elsewhere to search for $X^*$ in order to fulfill the accuracy requirement. Indeed, the worst-case query complexity is lower bounded by $\log(1/\epsilon)$ for any strategy that is $\epsilon$-accurate. However, the worst-case instance may not be one of $f_1,...,f_L$. To combine the $2L$ queries used to ensure privacy with the queries used to ensure accuracy therefore becomes the main challenge of the lower bound proof. To address this difficulty, we will again utilize Assumption~\ref{assump:complexity} on the richness of $\mathcal{F}$. On a high level, Assumption~\ref{assump:complexity} allows us to find a large class of functions in $\mathcal{F}$ which can also lead to the query sequence $\bar{q}$. Out of these functions, we show that for at least one of them it takes $\log(\delta/\epsilon)$ extra queries to search for its minimizer. Next we give the rigorous proof of the existence of such a function.

Firstly, note that $\bar{q}$ contains $L$ pairs of $\epsilon$-close queries that sandwich $x_1,...,x_L$. Since $\delta\geq \epsilon$, we have that for all $i$, $\bar{q}$ contains at least one query in $[x_i-\delta/2]$, and one query in $[x_i+\delta/2]$. 
Once at least one query has appeared in each of $[x_i-\delta/2,x_i]$ and $[x_i,x_i+\delta/2]$, we say $x_i$ is ``$\delta/2$-localized". Let $x_j$ be the last one to be $\delta/2$-localized out of $x_1,...,x_L$, and suppose it is $\delta/2$-localized at time $T$. Without loss of generality, assume a query in $[x_j-\delta/2,x_j]$ appears first, so that $\bar{q}_T\in [x_j,x_j+\delta/2]$. 
Let $I=[a,b]$ with $a$ defined as the query in $\bar{q}_1,...,\bar{q}_T$ to the left of $x_j$ that is the closest to $x_j$, and $b=x_j+\delta/2$. See Figure~\ref{fig:minimax.lower} for an illustration.

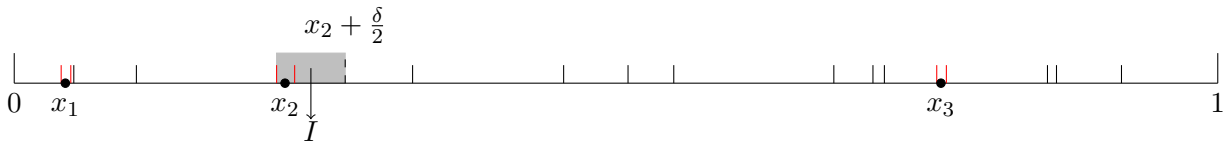
\begin{figure}[ht]
\begin{center}
   \begin{tikzpicture}[scale=0.8]
       \draw[fill=gray!50!white,draw=gray!50!white] (-5.64,0) rectangle ++(1.14,0.5);
    \draw (-10,0)-- (10,0); 
    \draw (-10,0.5)--(-10,0) node[below] {0};
    \draw (10,0.5)--(10,0) node[below] {1};
    \foreach \x in {-9.01,7.17,4.27,7.32,-7.97,0.2,8.4,-3.38,-0.87,4.46,3.62,0.96} {
        \draw (\x,0.3) -- (\x,0) node[below] {};
    }
    \foreach \x in {-9.22,-9.06,-5.64,-5.34,5.33,5.49} {
    	\draw[color=red] (\x,0.3) -- (\x,0) node[below] {};
    }
    \foreach \x in {-9.15,-5.5,5.4} {
    	\draw[fill=black] (\x,0) circle (0.07);
    }
    \node at (-9.15,-0.4) {$x_1$};
    \node at (-5.5,-0.4) {$x_2$};
    \node at (5.4,-0.4) {$x_3$};
    \draw[dashed] (-4.5,0) -- (-4.5,0.5) node[above] {$x_2+\frac{\delta}{2}$};
    \draw [->] (-5.07,0.25) -- (-5.07,-0.6);
    \node at (-5.07,-0.8) {$I$};
    \end{tikzpicture}
    \end{center}
    \caption{An illustration of the lower bound argument with $L=3$. The ticks represent all queries in $\bar{q}$. The $L$ pairs of $\epsilon$-close queries that sandwich $x_1,...,x_L$ are colored red. Suppose $x_2$ is the last one out of $x_1,...,x_L$ to be $\delta/2$-localized, and the query in $[x_2-\delta/2,x_2]$ appears before the one in $[x_2,x_2+\delta/2]$, then $I$ is defined as the shaded interval. Note that until all of $x_1,...,x_L$ are $\delta/2$-localized, no query is submitted in $I$.}
    \label{fig:minimax.lower}
   \end{figure}

Apply Lemma~\ref{lmm:complexity} with $I=[a,b]$, $f=f_j$ and $y=y_j$. We can find some $\widetilde{f}\in\mathcal{F}$ that satisfies the two criteria in the statement of Lemma~\ref{lmm:complexity}. Criterion (2) ensures that the gradient of $\widetilde{f}$ and $f_j$ coincide outside of $I$. Since $x_j$ is $\delta/2$-localized at time $T$, $\bar{q}_1,...,\bar{q}_{T-1}$ do not contain any queries between $a$ and $b$. Thus $q(\widetilde{f},y_j)$ and $q(f_j,y_j)=\bar{q}$ agree completely up to time $T-1$, and contain at least the $2L-1$ queries outside of $I$ used to sandwich $x_1,...,x_L$. The reason we need to subtract 1 is because the $T$'th queries in $\bar{q}$ is in $I$.


By criterion (1) in the statement of Lemma~\ref{lmm:complexity}, $q(\widetilde{f},y_j)$ contains at least $\log (|I|/\epsilon)\geq \log(\delta/(2\epsilon))$ queries in $I$. Combined with the $2L-1$ queries outside of $I$, we arrive at the desired lower bound $2L+\log(\delta/\epsilon)-2$.

\end{proof}

\medskip

\begin{proof}[Proof of Lemma~\ref{lmm:complexity}.]
The lemma is proved by constructing an $\widetilde{f}$ that satisfies both criteria. Our construction scheme in inspired by that of Nemirovski's (See Section 2.1.2 in lecture notes by Iouditski~\cite{iouditski2007efficient}). 
With the querying strategy $\phi$ fixed, we construct a sequence of functions $\{g_i\}_{i\geq 0}\subset \mathcal{F}$ adapted to the queries and the responses. The construction ensures that for each $i\geq 0$, there is an interval $\Delta_i\subset I$ with $|\Delta_i|\geq |I|/2^i$, such that 
\begin{enumerate}
\item  $g_i$ is minimized at the midpoint of $\Delta_i$;
\item in the query sequence $q(g_i,y)$, the first $i$ queries in $I$ are outside of $\Delta_i$.
\end{enumerate}


By Assumption~\ref{assump:complexity}, there exists a function in $\mathcal{F}$ whose gradient of $f$ agrees with that of $f$ outside of $I$, and is minimized at the midpoint of $I$. Let this function be $g_0$ and let $\Delta_0=I$. 

Inductively construct the rest of $\{g_i\}$. Given $g_0,...,g_{i}$, by the induction hypothesis in $q(g_i,y)$, the first $i$ queries in $I$ are all outside of $\Delta_{i}=[a_{i},b_{i}]$. Let $q$ be the $(i+1)$'th query of $q(g_{i},y)$ in $I$. If $q$ is not in $\Delta_{i}$, then we can simply let $g_{i+1}=g_{i}$ and $\Delta_{i+1}=\Delta_{i}$ to complete the $(i+1)$'th step of the induction. If $q\in\Delta_{i}$, depending on whether $q$ lands to the left or right of the midpoint of $\Delta_i$, let $\Delta_{i+1}$ be either $[q,b_i]$ or $[a_i,q]$, so that $|\Delta_{i+1}|\geq |\Delta_i|/2$.
Let $g_{i+1}\in \mathcal{F}$ be a function whose gradient agrees with $g_i$ outside of $\Delta_i$, and is minimized at the midpoint of $\Delta_{i+1}$. By Assumption~\ref{assump:complexity} such a $g_{i+1}$ always exists.

The construction can be carried out until for some integer $K$, we cannot find the $(K+1)$'th query of $q(g_K,y)$ in $I$. That is, $q(g_K,y)$ contains only $K$ queries in $I$. By construction, $q(g_K,y)$ does not contain any queries in $\Delta_K$. Therefore under Assumption~\ref{assump:complexity}, the learner cannot rule out any member of $\Delta_K$ being $X^*$. For the strategy to be $\epsilon$-accurate, we must have $|\Delta_K|<\epsilon$; hence $K>\log(|I|/\epsilon)$. Taking $\widetilde{f}=g_K$ finishes the proof of the lemma.
\end{proof}

\subsection{Proof under the Bayesian Setting}\label{sec:proof.Bayes}

\begin{proof}[Proof of the upper bound in Theorem~\ref{thm:bayes}]

Let $\nu$ denote the distribution of $X^*$. For an interval $I\subset[0,1]$, write $\nu_I$ for the probability distribution of $\nu$ conditioned on $I$, {\it i.e.}, $\frac{d\nu_I}{d\nu}(x)=\mathds{1}\{x\in I\}/\nu(I)$.
We design the following multi-phase querying strategy to attain the desired upper bound.

\begin{algorithm}[H]
\caption{Querying Strategy under the Bayesian Setting}
\label{algo:bayes}
\begin{algorithmic}[1]
\STATE 
Recursively query the median of the posterior distribution of $X^*$, until it is supported on an interval I with $\nu(I)\in [2\delta LH_\alpha, 4\delta LH_\alpha]$. 

\STATE 
Let $\kappa_j$ be the $j/L$ quantile of $\nu_I$ for $j = 0,1,...,L$
and let $I_{j}=[\kappa_{j-1}, \kappa_{j}]$ for $j \in [L]$. 
Query $\kappa_1,...,\kappa_{L-1}$ 
and identify $j^*$ for 
$f'(\kappa_{j^*-1}) \le 0$ and $f'(\kappa_{j^*}) > 0$ so that $I_{j^*}$ contains $X^*$.

\STATE 
Query the median $m_j$ of $\nu_{I_j}$ for $j\in [L]$. If $f'(m_{j^*})>0$, let $J_j=[\kappa_{j-1},m_j]$ for all $j$; otherwise let $J_j=[m_j,\kappa_j]$. 

\STATE 
For all $j\neq j^*$, sample $X_j\sim \nu_{J_j}$ independently. Denote $X_{j^*}=X^*$. For $j=1,...,L$, run the regular bisection search on $J_j$ to locate $X_j$ up to $\epsilon$-accuracy.
\end{algorithmic}
\end{algorithm}


Phase 1 runs the median-based bisection search, 
which is equivalent to the regular bisection search on $U=F_\nu(X^*)\sim \text{Unif}[0,1]$, where $F_\nu$ is the CDF of $\nu$. 
Note that this step is always possible under the assumption $2\delta L H_\alpha\leq 1$.
Phase 2 divides $I$ into $L$ subintervals  $I_1,...,I_L$ with equal $\nu$-probability and determines $I_{j^*}$ containing $X^*$.
Phase 3 is the key to ensure adequate separation between the subintervals $\{J_j\}_{j \in [L]}$. Phase 4 serves to achieve
the $\epsilon$-accuracy while obfuscating the adversary. See Figure~\ref{fig:bayes.strategy} for an illustration of phases 2 to 4.

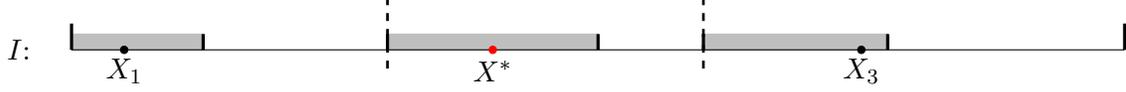
\begin{figure}
\begin{center}
   \begin{tikzpicture}[scale=0.7]
       \draw[fill=gray!50!white,draw=gray!50!white] (-10,0) rectangle ++(2.5,0.3);
       \draw[fill=gray!50!white,draw=gray!50!white] (-4,0) rectangle ++(4,0.3);
       \draw[fill=gray!50!white,draw=gray!50!white] (2,0) rectangle ++(3.5,0.3);
    \draw (-10,0)-- (10,0); 
    \draw[line width = 1.2] (-10,0.5)--(-10,0) node[below] {};
    \draw[line width = 1.2] (10,0.5)--(10,0) node[below] {};
    \draw[dashed, line width = 1] (-4,1)--(-4,-0.5) {};
     \draw[dashed, line width = 1] (2,1)--(2,-0.5) {};
    \foreach \x in {-4,2, -7.5,0,5.5} {
        \draw[line width = 1.2] (\x,0.3) -- (\x,0) node[below] {};
    }
    \draw[fill=black] (-9,0) circle (0.07);
    \draw[fill=black] (5,0) circle (0.07);
    \draw[fill=red,draw=red] (-2,0) circle (0.07);
    \node at (-9,-0.4) {$X_1$};
    \node at (-2,-0.4) {$X^*$};
    \node at (5,-0.4) {$X_3$};
    \node at (-11,0) {$I$:};
    \end{tikzpicture}
    \end{center}
    \caption{Example of phases 2 to 4 of the querying strategy under the Bayesian setting with $L=3$. In phase 2, the learner queries the $1/3$ and $2/3$ quantile of $\nu_I$ (represented by the dashed lines), and learns that $X^*\in I_2$. In phase 3, she queries the medians $m_1$,...$m_L$, and learners that $X^*$ is to the left of $m_2$. Therefore $J_1$,...,$J_L$ are defined to be the shaded intervals. In phase 4, $X_1$ and $X_3$ are sampled from $\nu_{J_1}$ and $\nu_{J_3}$ respectively and $X_2$ is defined to be $X^*$. Note that the separation of $X_1$,...,$X_L$ are guaranteed by the separation of $J_1$,...,$J_L$.}
    \label{fig:bayes.strategy}
   \end{figure}

The querying strategy outlined in Algorithm \ref{algo:bayes} is clearly $\epsilon$-accurate by design. We now show that it is also $(\delta,L)$-private. The high-level proof idea is to consider an adversary who has access to $X_1,...,X_L$. Using a genie-aided argument, we argue that this adversary is stronger than the one who only has access to the query sequence. We then establish that the conditional distribution of $X^*$ given $X_1,...,X_L$ is uniform on the $X_j$'s. Moreover, phase 3 of the querying strategy ensures that the $X_j$'s are all $\delta$-separated. Therefore  even with the additional knowledge of $X_1,...,X_L$, the adversary cannot estimate $X^*$ accurately with probability higher than $1/L$.

\medskip 
{\bf Proof of Privacy:} 
Since the adversary only has access to the query sequence $q$, any 
adversary's estimator $\widetilde{X}$ must be a (random) function of $q$,
that is $\wt X \equiv \wt X(q)$. 
Meanwhile by the design of our querying strategy, $q$ can be completely reconstructed from $X_1,...,X_L$. To see that, note that $I, \{I_j\},\{J_j\}$ and all the queries in phase 4 are deterministic functions of $X_1,...,X_L$. 
Therefore there is a mapping $\widetilde{\psi}$ such that $\wt X(q) = \widetilde{\psi}(X_1,...,X_L)$.
Thus,
\begin{align}
\mathbb{P}\left\{\left|\widetilde{X}-X^*\right|\leq \frac{\delta}{2}\right\}
& =\mathbb{E}\left[\mathbb{P}\left\{\left|\widetilde{X}(q)-X^*\right|\leq \frac{\delta}{2} \;\Big\rvert\;q \right\} \right] \nonumber \\ 
& \le\mathbb{E}\left[\sup_{\widetilde{\psi}}\mathbb{P}\left\{\left|\widetilde{\psi}(X_1,...,X_L)-X^*\right|\leq \frac{\delta}{2}\;\Big\rvert\; X_1,...,X_L\right\}\right] \nonumber \\
&\le  \mathbb{E}\left[\sup_{\tilde{x}\in [0,1]}\mathbb{P}\left\{\left|\tilde{x}-X^*\right|\leq \frac{\delta}{2}\;\Big\rvert\; X_1,...,X_L\right\}\right]. \label{eq:given.xj}
\end{align}

We claim that
\begin{enumerate}[(i)]
\item \label{enum:unif} $X^* \mid X_1,...,X_L\sim \text{Unif}\{X_1,...,X_L\}$.
\item \label{enum:delta.separated} With probability $1$, $|X_i-X_j|>\delta$ for all $i\neq j$. 

\end{enumerate}
Assuming the two claims hold,
$$
\sup_{\tilde{x}\in [0,1]}\mathbb{P}\left\{\left|\tilde{x}-X^*\right|\leq \frac{\delta}{2} \;\Big\rvert\; X_1,...,X_L \right\}
=\sup_{\tilde{x}\in [0,1]}\frac{1}{L}\sum_{j\leq L}\mathds{1}\left\{\left|\tilde{x}-X_j\right|\leq \frac{\delta}{2}\right\}\leq \frac{1}{L},
$$
where the equality is from~(\ref{enum:unif}) and the inequality is from~(\ref{enum:delta.separated}). Continuing~\eqref{eq:given.xj}, we have $\mathbb{P}\{|\widetilde{X}-X^*|\leq \delta/2\}\leq 1/L$. Thus our strategy is $(\delta,L)$-private. It remains to prove claims (\ref{enum:unif}), (\ref{enum:delta.separated}).

\smallskip

{\it Proof of (\ref{enum:unif}):} Recall that the index of the subinterval containing $X^*$ is $j^*$. Since $\nu(I_j)$ are equal for all $j$, $j^*$ is distributed uniformly in $\{1,...,L\}$. Therefore the desired claim $X^*\mid X_1,...,X_L\sim\text{Unif}\{X_1,...,X_L\}$ is equivalent to $j^*$ and $(X_1,...,X_L)$ being independent.

To show $j^*\indep (X_1,...,X_L)$, first note that $j^*\indep (J_1,...,J_L)$, because conditional on $j^*$, either $J_j=[\kappa_{j-1},m_j]$ for all $j$ 
or $J_j=[m_j,\kappa_j]$ for all $j$, with equal probability. 
Second, conditional on $(J_1, \ldots, J_L)$, $X_j$'s 
are independently distributed according to $\nu_{J_j}$ across all $j$. 
Therefore, we arrive at the conclusion $j^*\indep (X_1,...,X_L)$.

%
%
%

\smallskip

{\it Proof of (\ref{enum:delta.separated}):} It suffices to show that the intervals $J_1,...,J_L$ are $\delta$-separated, or equivalently, $|I_j\backslash J_j|\geq \delta$ for all $j\leq L$. Since phase 2 of the querying strategies queries all the medians of $I_1,...,I_L$, we have $\nu(I_j\backslash J_j)=\nu (I_j)/2=\nu(I)/(2L)\geq \delta H_\alpha$. Let $\mathbf{m}=d\nu/d\lambda$ be the density of $\nu$. Then
\begin{equation}
\label{eq:length.lower}
\left|I_j\backslash J_j\right|\geq \frac{\nu\left(I_j\backslash J_j\right)}{\sup_t \mathbf{m}(t)}=\frac{\delta H_\alpha}{\sup_t \mathbf{m}(t)}.
\end{equation}
To finish proof of this claim, we only need to bound the density of $\nu$ from above. Recall that $\nu$ is the distribution of $X^*$, which is the median of $F$. Thus the distribution function of $\nu$ has the form
\[
\nu([0,t])=\mathbb{P}\left\{X^*\leq t\right\}=\mathbb{P}\left\{F(t)\geq 1/2\right\}.
\]
Since $F\sim \text{DP}(\alpha,\lambda_{[0,1]})$, we have $(F(t),1-F(t))\sim \Dir(\alpha t,\alpha(1-t))$. Therefore $F(t)\sim \Beta(\alpha t,\alpha(1-t))$. We will use the following Lemma~\ref{lmm:beta} to bound the density of $\nu$. The proof of Lemma~\ref{lmm:beta} is deferred to the end of this subsection.
\begin{lemma}
\label{lmm:beta}
Suppose $X\sim \Beta(\alpha t,\alpha(1-t))$ for some $\alpha>0$, then for all $t\in (0,1)$,
\[
h_\alpha \leq \frac{d}{dt}\mathbb{P}\left\{X\geq 1/2\right\} \leq H_\alpha,
\]
where $h_\alpha = \tfrac{1}{3}2^{-\alpha-2}$ and $H_\alpha= (3+2e^{-1})\alpha+14$.
\end{lemma}

By Lemma~\ref{lmm:beta},
\begin{equation}
\label{eq:mu.pdf}
\mathbf{m}(t)=\frac{d}{dt}\mathbb{P}\left\{F(t)\geq 1/2\right\}\leq H_\alpha,
\end{equation}
for all $t\in [0,1]$. Combining~\eqref{eq:length.lower} and~\eqref{eq:mu.pdf} yields that
\[
\left|I_j\backslash J_j\right|\geq \frac{\delta H_\alpha}{H_\alpha}\geq \delta.
\]
{We have shown that $\nu_{j_1}$, ... $\nu_{j_L}$ are continuous distributions supported on $L$ intervals that are $\delta$-separated from each other. Therefore $|X_i-X_j|>\delta$ for all $i\neq j$ with probability $1$.}

\smallskip

{\bf Query Complexity:} The number of queries submitted in phase 1 is at most $\log (1/(2\delta LH_\alpha))$. Phase 2 and phase 3 involve $L-1$ and $L$ queries respectively. The number of queries submitted in phase 4 equals
\[
\sum_{j\leq L}\left\lceil\log\frac{|J_j|}{\epsilon}\right\rceil\leq L+\sum_{j\leq L}\log\frac{|J_j|}{\epsilon} = L+\log\left(\prod_{j\leq L}|J_j|\right)+L\log\frac{1}{\epsilon},
\]
To bound the above, note that from Lemma~\ref{lmm:beta} we have
\[
\sum_{j\leq L}|J_j|\leq \frac{\nu(\cup_{j\leq L}J_j)}{h_\alpha} \leq \frac{2\delta L H_\alpha}{h_\alpha}.
\]
Therefore $\prod_{j\leq L}|J_j|\leq (2\delta H_\alpha/h_\alpha)^L$. Thus the total number of queries submitted by the learner is at most
\begin{align*}
&\log\frac{1}{2\delta L H_\alpha}+(L-1)+L+L\left(\log\frac{\delta}{\epsilon} + \log\frac{4H_\alpha}{h_\alpha}\right)\\
=& L\left(\log\frac{\delta}{\epsilon}+\log\frac{16H_\alpha}{h_\alpha}\right) +\log\frac{1}{\delta L}+\log\frac{1}{4H\alpha}\\
\leq & L\left(\log\frac{\delta}{\epsilon} + c_2\right)+ \log\frac{1}{\delta L}
\end{align*}
for $c_2=\log(16H_\alpha/h_\alpha)$. The inequality is from $H_\alpha>14$ for all $\alpha>0$. 
\end{proof}

\medskip

\begin{proof}[Proof of the lower bound in Theorem~\ref{thm:bayes}]
Let $\phi$ be a querying strategy that is both $\epsilon$-accurate and $(\epsilon,L)$-private. By definition of $(\epsilon,L)$-privacy, we must have for any adversary's estimator $\widetilde{X}$,
\begin{equation*}
\frac{1}{L}\geq \mathbb{P}\left\{\widetilde{X}\in [X^*-\delta/2,X^*+\delta/2]\right\}.
\end{equation*}
For the purpose of the lower bound, we can assume without loss of generality that the learner always submits a fixed $n$ number of queries under strategy $\phi$. If the lengths of the query sequences $q(f^*,Y)$ depend on $f^*$ and $Y$, the learner can always fill the short sequences with $n-|q(f^*,Y)|$ trivial queries at 0 without hurting the accuracy or the privacy of learning.

Next we complete the lower bound proof following the outline given in Section~\ref{sec:proof.overview}.

{\bf Step 1:} Quantify the learner's information using learner's intervals. Recall that the $i$'th learner's interval $I_i$ denotes the smallest interval that the learner knows to contain $X^*$. 

{\bf Step 2:} Analyze the conditional distribution of $X^*$ over the learner's interval. To find a ``good" event $\mathcal{B}$ on which the conditional distribution is uniform, we heavily rely on the stick-breaking characterization of the Dirichlet Process. Namely, the event $\mathcal{B}$ is associated with the length of the longest stick in the stick-breaking process. For completeness, we shall include a brief description of the stick-breaking process here. 

Given base distribution $\mu_0$ and scaling parameter $\alpha>0$, draw $\{X_k\}_{k=1}^\infty$ i.i.d.\ from $\mu_0$, and independently draw $\{V_k\}_{k=1}^\infty$ i.i.d.~from $\Beta(1,\alpha)$. From a stick of unit length, break off the first stick of length $V_1$; break off $V_2$ fraction of the remaining stick and repeat. In other words, denote by $\beta_k$ the length of the $k$'th stick. We have
\[
\beta_k=V_k\cdot \prod_{j\leq k-1}\left(1-V_k \right)
\]
and $\sum_{k=1}^\infty \beta_k=1$. 
Let $\mu=\sum_{k\geq 1}\beta_k \delta_{X_k}$ be the discrete distribution supported on $\{X_k\}_{k=1}^\infty$, where $\delta_{X_k}$ denotes the point mass distribution at $X_k$. 
Then $\mu$ with the distribution function of $F$ follows the Dirichlet process $\text{DP}(\mu_0,\alpha)$.

Here is a heuristic argument on how the stick-breaking process helps us prove the uniformity of the conditional distribution of $X^*$. Under our prior construction, $X^*$ is at the median of $F\sim \text{DP}(\lambda_{[0,1]},\alpha)$, where we recall that $\lambda_{[0,1]}$ is the Lebesgue measure on $[0,1]$.
Therefore, $X^*$ occurs at one of the stick-breaking locations $X_k$. Even though the $X_k$'s are distributed {\it i.i.d.} uniformly in $[0,1]$, $X^*$ itself does not follow the uniform distribution since the index $i$ that corresponds to $X^*$ is random. The key observation is that the conditional distribution of $X^*$ is uniform conditional on the event $\mathcal{A}$ where the length of the longest stick is at least $1/2$. To prove uniformity, we first show that on the event $\mathcal{A}$, the median $X^*$ must occur at the $X_k$ that corresponds to the longest stick. Moreover, by independence of the stick lengths $\{\beta_k\}_{k\geq 1}$ and the locations $\{X_k\}_{k\geq 1}$, the distribution of the location corresponding to the longest stick is uniform in $[0,1]$. Furthermore, the posterior distribution of $X^*$ remains uniform as queries are sequentially submitted. The following Lemma~\ref{lmm:uniform} contains the precise statement on uniformity. 

Some notation is necessary before  stating Lemma~\ref{lmm:uniform}. Firstly, denote by $\beta_{(1)}, \beta_{(2)},...$ the order statistics of the lengths of the sticks in the stick-breaking process corresponding to $F$. Let
\[
\mathcal{A}= \left\{\beta_{(1)}\geq 1/2\right\}=\cup_{z\geq 1/2} \mathcal{A}_z,\;\;\;\text{ where }\mathcal{A}_z=\left\{\beta_{(1)} = z\right\}.
\]
Let $J\subset [0,1]$ be an arbitrary fixed interval. Write $[q_-,q_+]=I_i\cap J$. Let the event $\mathcal{B}=\mathcal{B}(z,J,y,i,\rho^{(i)},\rho_-,\rho_+)$ encode the random instances of $F$, $Y$ and the first $i$ responses, defined as
\begin{equation}\label{eq:def.B}
\mathcal{B} = \left\{\mathcal{A}_z,\, X^*\in J, \, Y=y, \,  r^{(i)}=\rho^{(i)}, \, F(q_-)=\rho_-, \,F(q_+)=\rho_+\right\}.
\end{equation}
See Figure~\ref{fig:Bayes.lower} for an example of $F$ and the quantities in~\eqref{eq:def.B}. 

\begin{lemma}\label{lmm:uniform}
For all $z\geq 1/2,$ $J,$ $y,$ $i,$ $\rho^{(i)},$ $\rho_-<1/2,$ $\rho_+>1/2$, we have for $\mathcal{B}$ defined in~\eqref{eq:def.B},
\[
\mathcal{L}\left(X^*\mid \mathcal{B}\right)=\mathrm{Unif}[q_-,q_+],
\]
where $\mathcal{L}(\cdot)$ denotes the (conditional) distribution.
\end{lemma}

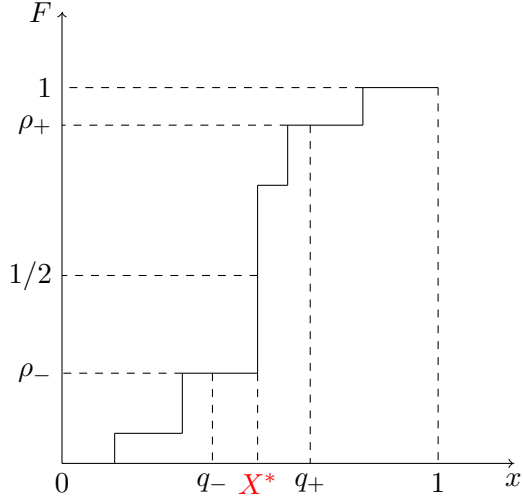
\begin{figure}
\begin{center}
   \begin{tikzpicture}[scale=1]
	\draw (0.7,0.4)--(0.7,0) node {};
	\draw (0.7,0.4)--(1.6,0.4) node {};
	\draw (1.6,1.2)--(1.6,0.4) node {};
	\draw (1.6,1.2)--(2.6,1.2) node {};
	\draw (2.6,3.7)--(2.6,1.2) node {};
	\draw (2.6,3.7)--(3,3.7) node {};
	\draw (3,4.5)--(3,3.7) node {};
	\draw (3,4.5)--(4,4.5) node {};
	\draw (4,4.5)--(4,5) node {};
	\draw (5,5)--(4,5) node {};
	\draw[dashed] (4,5)--(0,5) node[left] {$1$};
	\draw[dashed] (3,4.5)--(0,4.5) node[left] {$\rho_+$};
	\draw[dashed] (1.6,1.2)--(0,1.2) node[left] {$\rho_-$};
	\draw[dashed] (2.6,2.5)--(0,2.5) node[left] {$1/2$};
	\draw[dashed] (2.6,1.2)--(2.6,0) node[below,red] {$X^*$};
    \draw [->] (0,0) -- (0,6);
    \draw [->] (0,0) -- (6,0);
    \draw[dashed] (2,1.2)--(2,0) node[below] {$q_-$};
    \draw[dashed] (3.3,4.5)--(3.3,0) node[below] {$q_+$};
    \draw[dashed] (5,5)--(5,0) node[below] {$1$};
    \draw (0,0)--(0,0) node[below] {$0$};
    \draw (6,0)--(6,0) node[below] {$x$};
    \draw (0,6)--(0,6) node[left] {$F$};
    
        \end{tikzpicture}
    \end{center}
    \caption{An illustration of the quantities in~\eqref{eq:def.B}. 
    Conditional on $X^*\in J$ and the responses to the first $i$ queries, the range of $X^*$ is narrowed down to $I_i\cap J=[q_-,q_+]$. Further conditioning on $F(q_-)=\rho_-$ and $F(q_+)=\rho_+$, we show that $F$ restricted to $[q_-,q_+]$ also follows a Dirichlet process after appropriate scaling.}
    \label{fig:Bayes.lower}
   \end{figure}

The proof of Lemma~\ref{lmm:uniform} is deferred to the end of this subsection. It utilizes the self-similarity property of the Dirichlet process. See Section~\ref{sec:self.similar} in the appendix for a description and proof of the self-similarity property. In short, it ensures that the values of $F$ inside of $[q_-,q_+]$ conditional on information outside of $[q_-,q_+]$ also follows a scaled Dirichlet process. Thus the learner cannot gain too much information about the location of $X^*$ in $[q_-,q_+]$.

{\bf Step 3:} Control the speed at which the learner's interval shrinks. Heuristically, since the conditional distribution of $X^*$ stays uniform over the learner's interval in view of Lemma~\ref{lmm:uniform}, the learner cannot search faster than the bisection method, and the learner's interval cannot shrink faster than $1/2$ each time a query is submitted.

Recall that $[0,1]$ is divided into $2/\delta$ subintervals $J_1,...,J_{2/\delta}$ of length $\delta/2$, and $J^*$ denotes the subinterval of contains $X^*$. In this step, by integrating over instances of $\mathcal{B}$, and letting $J$ range over the $2/\delta$ subintervals, we prove the following lemma.

\begin{lemma}
\label{lmm:learner.interval}
For all $i$, we have that 
\begin{equation}
\label{eq:learner.interval}
\mathbb{E}\left(\log \frac{|I_{i+1}\cap J^*|}{|I_i\cap J^*|}\Big\rvert \mathcal{A}\right)\geq -\mathbb{P}\left\{q_{i+1}\in J^*\mid \mathcal{A}\right\}.
\end{equation}
\end{lemma}
The proof of Lemma~\ref{lmm:learner.interval} is deferred to the end of this subsection.

{\bf Step 4:} In this step, we apply Lemma~\ref{lmm:learner.interval} to obtain the desired lower bound on the optimal query complexity. By writing $\log |I_n\cap J^*|$ as a telescoping sum, we have that 
\begin{align*}
\mathbb{E}\left(\log |I_n\cap J^*|\mid \mathcal{A}\right) = &\log |I_0\cap J^*| + \sum_{i=0}^{n-1}\mathbb{E}\left(\log \frac{|I_{i+1}\cap J^*|}{|I_i\cap J^*|}\Big\rvert \mathcal{A}\right)\\
=&\log \frac{\delta}{2} + \sum_{i=0}^{n-1}\mathbb{E}\left(\log \frac{|I_{i+1}\cap J^*|}{|I_i\cap J^*|}\Big\rvert \mathcal{A}\right)\\
\geq & \log\frac{\delta}{2} - \mathbb{E}\left(\text{number of queries in }J^*\mid \mathcal{A}\right).
\end{align*}
Therefore, on the one hand, by Jensen's inequality,
\begin{equation}
\label{eq:In.lower}
\log \mathbb{E}(|I_n\cap J^*|\mid \mathcal{A}) \geq \mathbb{E}\left(\log |I_n\cap J^*|\mid \mathcal{A}\right)
\geq \log\frac{\delta}{2}- \mathbb{E}\left(\text{number of queries in }J^*\mid \mathcal{A}\right).
\end{equation}
On the other hand, from the accuracy requirement, we must have $|I_n|\leq \epsilon$ with probability 1. 
Therefore 
\begin{equation}
\label{eq:In.upper}
\mathbb{E}\left(|I_n\cap J^*|\;\Big\rvert\; \mathcal{A}\right)\leq \mathbb{E}\left(|I_n|\;\Big\rvert\;\mathcal{A}\right)\leq \epsilon/2.
\end{equation}
Combining~\eqref{eq:In.lower},~\eqref{eq:In.upper} yields
\begin{equation}
\label{eq:queries.in.Jstar}
\mathbb{E}\left(\text{number of queries in }J^*\mid \mathcal{A}\right) \geq \log\frac{\delta}{\epsilon}.
\end{equation}

Consider an adversary who adopts the {\it proportional-sampling} strategy~\cite{xu2018query}. That is, suppose the adversary's estimator $\widetilde{X}$ is sampled from the empirical distribution of the queries. For this particular $\widetilde{X}$,
\[
\mathbb{P}\left\{\widetilde{X}\in [X^*- \delta/2,X^*+\delta/2]\right\}=\frac{\mathbb{E}(\text{number of queries in }[X^*-\delta/2, X^*+\delta/2])}{n},
\]
which gives a lower bound on the total number of queries:
\begin{equation}
\label{eq:prop.sample}
n\geq L\mathbb{E}(\text{number of queries in }[X^*-\delta/2, X^*+\delta/2]).
\end{equation}
Since $J^*\subset [X^*-\delta/2, X^*+\delta/2]$, it follows from~\eqref{eq:prop.sample} and~\eqref{eq:queries.in.Jstar} that
\begin{equation}
\label{eq:prop.sample.1}
N(\epsilon,\delta,L)\geq L\mathbb{E}(\text{number of queries in }J^*) \geq \mathbb{P}\left(\mathcal{A}\right)L\frac{\delta}{\epsilon}.
\end{equation}
We have thus arrived at the desired query complexity lower bound with
\[
c_1=\mathbb{P}(\mathcal{A})=\mathbb{P}\left\{\beta_{(1)}>1/2\right\} \geq \mathbb{P}\left\{\beta_{1}> 1/2\right\},
\]
where $\beta_1\sim \Beta(1,\alpha)$ is the length of the first stick fom the stick-breaking characterization of the Dirichlet process. The completes the proof of the Bayesian lower bound. \end{proof}

\begin{proof}[Proof of Lemma~\ref{lmm:uniform}]
Since the gradient of the convex function $f^*$ is defined with $(f^*)'=\gamma_-+(\gamma-\gamma_-)F$, the minimizer of $f^*$ is at the median of $F$, i.e.,
\[
X^* = \inf\left\{x: F(x)\geq \frac{-\gamma_-}{\gamma_+-\gamma_-}=\frac{1}{2}\right\}.
\]
Under our prior construction, the distribution of $F$ follows a Dirichlet process with the uniform base distribution on $[0,1]$ and scale parameter $\alpha$. Therefore with probability 1, $F$ is a distribution function with countably many points of discontinuity, which we will refer to as {\it jumps}. If we characterize $F$ with the stick breaking process, then the locations of the jumps are at $X_1,X_2,...$ where the $X_k$'s are independently and uniformly distributed on $[0,1]$. The sizes of the jumps $\beta_1,\beta_2,...$ correspond to the lengths of the sticks from the stick-breaking process. We have $\sum \beta_k=1$, and the two sequences $\{X_k\}_{k\geq 1}$ and $\{\beta_k\}_{k\geq 1}$ are independent.

To proceed, we first show that if the size of the largest jumps is larger than $1/2$, then $X^*$ must occur at the largest jump. That is,
\begin{equation}
\label{eq:long.stick}
\mathcal{A}\subset \cup_{i\geq 1}\left\{X^*=X_k,\beta_{(1)}=\beta_k\right\}.
\end{equation}
To see why, recall that $X^*$ is the median of $F$. Thus $F(X^*)\geq 1/2$ and $\sup_{x<X^*}F(x)\leq 1/2$. Suppose $\beta_{(1)}=\beta_k$. We consider two cases: 
\begin{enumerate}
\item if $X^*<X_k$, then $F(X_k)\geq F(X^*)+\beta_{(1)}>1$; 
\item if, on the other hand, $X^*>X_k$, then $F(X_k)\leq \sup_{x<X^*}F(x)-\beta_{(1)}\le 1/2-\beta_{(1)}<0$.
\end{enumerate}
In neither case can $F$ be a distribution function. Therefore we must have $X^*=X_k$ is the location of the largest jump. 

For $z\geq 1/2$, conditional on $\mathcal{A}_z$ and $X^*\in [q_-,q_+]$, we know that $X^*$ is at the largest jump in $[q_-,q_+]$.  Moreover, since the learner would not have submitted any queries between $q_-$ and $q_+$ at time $i$, the events conditioned on do not contain any information on the location of the largest jump. Therefore the conditional distribution of $X^*$ is uniform. To prove the claim rigorously, we need to invoke the self-similarity property of the Dirichlet process.

Recall that $F$ follows a Dirichlet Process is supported on $[0,1]$ with base distribution $\lambda_{[0,1]}$. The self-similarity property 
asserts that
for any finite partition $0=x_0\leq x_1\leq...\leq x_{n-1}\leq x_n=1$ of $[0,1]$, conditional on the realization of $F$ on $x_1,...,x_n$, the restriction of $F$ onto each subinterval is also a Dirichlet process scaled. In particular, for each $j\leq n$, we have
\[
\mathcal{L}\left(\frac{[F]_{[x_j,x_{j+1}]}-t_j}{t_{j+1}-t_j} \;\Big\rvert\; F(x_1)=t_1, ...,F(x_{n-1})=t_{n-1}\right)
=\text{DP}\left(\lambda_{[x_{j},x_{j+1}]}, \alpha \lambda\left(\left[x_j,x_{j+1}\right]\right)\right),
\]
where $[F]_{I}$ denotes the function $F$ restricted to interval $I$, $\lambda_I$ denotes the uniform probability measure on $I$, and $\lambda(I)$ denotes the Lebesgue measure of $I$.
This property is well-known, and follows from the definition of the Dirichlet process. See Section~\ref{sec:self.similar} in the appendix for a proof.

Importantly, the following is a direct consequence of the self-similarity property. For each interval $[a,b]\subset[0,1]$, conditional on the value of $F(a)$ and $F(b)$, the distribution of $F$ restricted to $[a,b]$ is independent of the realization of $F$ outside of $[a,b]$. As a result, for each interval $I\subset [0,1]$, given $X^*\in I$, the learner cannot gain any additional information on $X^*$ without querying in $I$. This property ensures that the posterior distribution of $X^*$ conditional on $\mathcal{A}$ and the responses is uniform between the two closest queries that sandwich $X^*$. Therefore, the learner cannot beat the bisection search on the event $\mathcal{A}$.

By definition of the learner's interval $I_i$, none of the first $i$ queries $q_1,...,q_i$ can be in $I_i\cap J_j=[q_-,q_+]$. Since $X^*$ is determined by the values of $F$ inside $[q_-,q_+]$, by the self-similarity property of the Dirichlet process, $X^*$ is independent of the responses to the first $i$ queries conditioning on the values of $F(q_-)$ and $F(q_+)$. Therefore the event $\{r^{(i)}=\rho^{(i)}\}$ can be dropped from $\mathcal{B}$ without changing the conditional distribution of $X^*$. The indicator $\mathds{1}\{X^*\in J\}$ is completely determined by whether $\rho_-$ and $\rho_+$ are above or below $1/2$;
and the outside randomness $Y$ is independent of $F$. Therefore we can drop both events $\{X^*\in J\}$ and $\{Y=y\}$, and obtain
\[
\mathcal{L}\left(X^* \mid \mathcal{B}\right)
=\mathcal{L}\left(X^* \mid \mathcal{A}_z,F(q_-)=\rho_-,F(q_+)=\rho_+\right).
\]

By the self-similarity property of the Dirichlet process, given $F(q_-)=\rho_-$ and $F(q_+)=\rho_+$,the conditional distribution of $(F-\rho_-)/(\rho_+-\rho_-)$ restricted to $[q_-,q_+]$ is also a Dirichlet process with the uniform base distribution on $[q_-,q_+]$ and scaling parameter $\alpha'=\alpha(q_+-q_-)$. In other words, there exist ancillary random vectors $\{X_k'\}_{k\geq 1}$, $\{\beta_k'\}_{k\geq 1}$ generated from a stick-breaking process that characterize the distribution function
$$
\widetilde{F}=(F-\rho_-)/(\rho_+-\rho_-)
$$
on $[q_-,q_+]$. In addition, $X_k'\stackrel{i.i.d.}{\sim}\text{Unif}[q_-,q_+]$, and $(\{X_k'\}_{k\geq 1},\{\beta_k'\}_{k\geq 1})$ 
is independent of $(F(q_-),F(q_+))$.

We claim that for all $z\geq 1/2$, the event $\mathcal{A}_z=\{\beta_{(1)}=z\}$ is equivalent to $\{\beta_{(1)}'=z/(\rho_+-\rho_-)\}$. Suppose $\mathcal{A}_z$ holds, and say $\beta_{(1)}=\beta_j$. Then by~\eqref{eq:long.stick}, $X^*=X_j$. Thus $[q_-,q_+]$ contains the largest jump in $F$. Since $\widetilde{F}$ is a scaled version of $F$ restricted to $[q_-,q_+]$, the largest jump of $\widetilde{F}$ must be of size $z/(\rho_+-\rho_-)$. Conversely, if $\beta_{(1)}'=z/(\rho_+-\rho_-)$, then $F$ contains a jump of size $z$. When $z\geq 1/2$, this must be the largest jump in $F$, i.e. $\beta_{(1)}=z$.

Note that conditional on $\mathcal{A}_z$ for $z \ge 1/2$, $X^*$ can be written as the location of the largest jump in $\widetilde{F}$. We have shown that $X^*$ and $\mathcal{A}_z$ can both be expressed as functions that only depend on $\{X_k',\beta_k'\}$. As a result,
\begin{align*}
&\mathcal{L}(X^* \mid \mathcal{A}_z,F(q_-)=\rho_-,F(q_+)=\rho_+)\\
= & \mathcal{L}\left(\text{location of the largest jump in } F' \mid \beta_{(1)}'=\frac{z}{\rho_+ - \rho_-}, F(q_-)=\rho_-,F(q_+)=\rho_+\right)\\
\stackrel{(a)}{=} & \mathcal{L}\left(\text{location of the largest jump in } F' \mid \beta_{(1)}'=\frac{z}{\rho_+ - \rho_-}\right)\\
\stackrel{(b)}{=} &\mathcal{L}\left(\text{location of the largest jump in } F' \right)\\
\stackrel{(c)}{=}&\mathcal{L}(X_1')=\text{Unif}[q_-,q_+],
\end{align*}
where (a) is from the independence between $(\{X'_k\}_{k\geq 1},\{\beta'_k\}_{k\geq 1})$ and $(F(q_-),F(q_+))$; (b) holds because by the stick-breaking characterization of the Dirichlet process, the locations of the jumps $\{\beta_k\}_{k\geq 1}$ and the sizes of the jumps $\{X_k\}_{k\geq 1}$ are independent. More specifically, let $j$ be the index of the largest jump, i.e., $\beta_{(1)}'=\beta_j'$. Then $j$ is only a function of $\{\beta_k\}_{k\geq 1}$ and is therefore independent of $\{X_k'\}_{k\geq 1}$. We have $X_j'$ is independent of $\{\beta_k'\}_{k\geq 1}$, thus we can drop the conditional event which only depends on $\{\beta_k'\}_{k\geq 1}$; (c) is again from the independence of $j$ and $\{\beta_k'\}_{k\leq 1}$. Since $\{X_k'\}_{k\geq 1}$ are distributed {\it i.i.d.} Unif $q_-,q_+$, we have $\mathcal{L}(X_j')=\mathcal{L}(X_1')=\text{Unif}[q_-,q_+]$. 

\end{proof}

\begin{proof}[Proof of Lemma \ref{lmm:learner.interval}]

From Lemma~\ref{lmm:uniform}, we have $\mathcal{L}(X^*\mid \mathcal{B})=\text{Unif}[I_i\cap J]$. We first claim that as a consequence,
\begin{equation}
\label{eq:In.lower.det}
\mathbb{E}\left(\log \frac{|I_{i+1}\cap J|}{|I_i\cap J|}\Big\rvert \mathcal{B} \right)
 \geq -\mathds{1}\left\{q_{i+1}=\phi_i(\rho^{(i)}, y)\in I_i\cap J\right\}.
\end{equation}
The inequality \eqref{eq:In.lower.det}  can be interpreted as follows. Firstly, the interval $I_{i}\cap J^*$ is only shortened when querying within $I_{i}\cap J^*$. Secondly, conditional on all instances of the behavior of $F$ outside of $I_{i}\cap J^*$, on average, no query can reduce the length of $I_{i}\cap J^*$ by more than a half. 

By taking the union of the events $\mathcal{B}$ over all the variables $z> 1/2$, $y\in [-0,1]$, $\rho_-<1/2$, $\rho_+>1/2$, $\rho^{(i)}$, and $J$ ranging over $J_1,...,J_{2/\delta}$, we arrive at the event $\mathcal{A}$. Therefore, 
integrating \eqref{eq:In.lower.det} over these variables yields that 
\begin{equation*}
\mathbb{E}\left(\log\frac{|I_{i+1}\cap J^*|}{|I_i\cap J^*|} \, \Big\rvert \, \mathcal{A}\right) \geq -\mathbb{P}\left\{q_{i+1}\in I_i\cap J^*\mid \mathcal{A}\right\}\geq -\mathbb{P}\left\{q_{i+1}\in J^*\mid \mathcal{A}\right\}.
\end{equation*}

It remains to verify~\eqref{eq:In.lower.det}. If $q_{i+1}\notin I_i\cap J$, then $I_{i+1}\cap J=I_i\cap J$ and the claim~\eqref{eq:In.lower.det} trivially holds. If $q_{i+1}\in I_i\cap J$, we have
\[
\log \frac{|I_{i+1}\cap J|}{|I_i\cap J|} 
= \mathds{1}\{X^*\leq  q_{i+1}\} \log\frac{q_{i+1}-q_-}{q_+-q_-} + \mathds{1}\{X^*> q_{i+1}\}\log\frac{q_+-q_{i+1}}{q_+-q_-}.
\]
Since the conditional distribution of $X^*$ is uniform, we have
\[
\mathbb{E}\left(\log \frac{|I_{i+1}\cap J|}{|I_i\cap J|}\Big\rvert \mathcal{B}\right)
\geq 
\inf_{t\in [0,1]}[t\log t+(1-t)\log (1-t) ]=-1.
\]
We have finished the proof of~\eqref{eq:In.lower.det} and, by consequence, Lemma \ref{lmm:learner.interval}. 

\end{proof}

\medskip

\section{Extension to Multidimensions}
\label{sec:ddim}
In this section we extend our results under the minimax setting to optimization of convex separable functions in $\mathbb{R}^d$. 
Separable convex optimization arises in a variety applications such as inventory control in operation research,
resource allocation in networking, and distributed optimization in multi-agent networks~\cite{nedic2008convex,padakandla2010separable,boyd2011distributed}, when the global objection function is a sum of
the local objective functions and each local objective function depends only on one component of the decision variable. 
Here, separability ensures that there is no cross-coordinate information leakage. 
Further generalizing our result to allow for general (non-separable) functions in $\mathbb{R}^d$ is left as future work.

Suppose the true function $f^*:[0,1]^d\rightarrow\mathbb{R}$ belongs to a family of convex separable functions
\[
\mathcal{F} = \left\{f: f(x)=\sum_{i=1}^d f_i(x_i), \, f_i\in\mathcal{F}_i\right\},
\]
where each $\mathcal{F}_i$ is a family of one-dimensional convex functions. 
For each query $q\in [0,1]^d$ submitted, the learner receives the gradient vector $\nabla f(q)=(f_1'(q_1),...,f_d'(q_d))$ as the response. We say a querying strategy is $\epsilon$-accurate if
\[
\inf_{f\in \mathcal{F}} \mathbb{P}_f\left\{\left\|\widehat{X}-x\right\|_\infty\leq \epsilon/2\right\}=1,
\]
where $x$ is the minimizer of $f$. 
We say $\phi$ is $(\delta,L)$-private if 
\[
\sup_{\widetilde{X}}\inf_{f\in\mathcal{F}} \mathbb{P}_f\left\{\left\|\widetilde{X}-x\right\|_\infty\leq \delta/2\right\}\leq 1/L.
\]
In other words, we declare privacy breach if the adversary's estimator is within a $\delta/2$-neighborhood around the true minimizer with probability higher than $1/L$. 
As in the one-dimensional case, we need to impose some assumption on the complexity of the function class $\mathcal{F}$. Since $\mathcal{F}$ contains only separable functions, we can simply impose the one-dimensional assumption onto each of the $d$ one-dimensional function classes $\mathcal{F}_1,...,\mathcal{F}_d$.
Below is the extension of our one-dimensional result to $d$ dimensions.

\begin{theorem}
Let $N_d(\epsilon,\delta,L)$ denote the optimal query complexity in dimension $d$ under the minimax setting. Suppose $\mathcal{F}_i$ all satisfy Assumption~\ref{assump:complexity} for all $i=1,...,d$. If $2\epsilon\leq \delta\leq L^{-1/d}$, then
\[
2L^{1/d}+\log\frac{\delta}{\epsilon}-2
\leq N_d(\epsilon,\delta,L) \leq  
\begin{cases}
2L^{1/d}+\log\frac{\delta}{\epsilon} & \text{ if } L^{1/d} \ge \log \frac{1}{\delta} \\
 L^{1/d}+ \log\frac{1}{\epsilon} & \text{ o.w.} 
 \end{cases}
 \, .
\]
\end{theorem}

\begin{remark}
We choose to quantify the error of the learner and the adversary with respect to the $\|\cdot\|_\infty$ norm because $\|x-y\|_\infty\leq \epsilon/2$ is equivalent to $|x_i-y_i|\leq \epsilon/2$ for all $i\leq d$, so the analysis can be elegantly reduced to the one-dimensional case. However our result does not crucially depend on the choice of the norm. From the basic inequality $\|x\|_\infty\leq \|x\|_2\leq \sqrt{d}\|x\|_\infty$, we have that the optimal query complexity can differ by at most a $d$-dependent additive constant if the Euclidian norm were used instead.
\end{remark}
\begin{proof}[Proof of the upper bound.]
Under the minimax privacy framework, to make a strategy private, we only need to find $L$ functions $f^{(1)},..,f^{(L)}\in \mathcal{F}$ whose minimizers are $\delta$-apart, such that the query sequence for $f^{(1)},...,f^{(L)}$ are identical. That would ensure that the adversary who only observes the query sequence cannot succeed with probability higher than $1/L$. 

To construct such $L$ functions, we design a querying strategy that submits $L^{1/d}$ {\it guesses} $\delta$-apart along each dimension. To recap, in Section~\ref{sec:proof.overview} we defined a guess at $x$ to be a pair of $\epsilon$-apart queries $(x,x+\epsilon)$. 
The guesses across the $d$ dimensions intersect with each other in $[0,1]^d$ to create $(L^{1/d})^d=L$ cubes of diameter $\epsilon$ that potentially contain the minimizer of the true function $f^*$. The guesses are submitted following the same algorithm as in the one-dimensional case (see the upper bound proof of Theorem~\ref{thm:minimax}), except with $L$ replaced by $L^{1/d}$. 

Note that since each query is a $d$-dimensional vector and the function $f^*$ is separable, we can run the search algorithms along the $d$ directions in parallel.
More concretely, write $f^*(x)=\sum_{i\leq d}f_i^*(x_i)$, and let $q=(q_1,q_2,...,q_n)$ be the query sequence where $q_j=(q_{j,1},...,q_{j,d})\in[0,1]^d$. Each time the learner submits a query $q_j$, she receives the gradient vector
\[
\nabla f^*(q_j) = \left((f_1^*)'(q_{j,1}), ..., (f_d^*)'(q_{j,d})\right).
\]
For each dimension $i$, the learner leverages the gradient information $(f_i^*)'(q_{j,i})$ and constructs the next query $q_{j+1,i}$ in dimension $i$, as if she were learning the minimizer of $f_i^*$ in one-dimension.

In particular, fix any dimension $1 \le i \le d$.
The first $2L^{1/d}$ queries $q_{1,i},...,q_{2L^{1/d},i}$ consist of $L^{1/d}$ pairs of queries (guesses) that are $\delta$-apart. When $\delta \leq 2^{-L^{1/d}}$, 
these guesses are submitted along the bisection search path: 
\begin{enumerate}
\item The first guess is at $1/2$, i.e., $q_{1,i}=1/2$ and $q_{2,i}=1/2+\epsilon$. The learner's interval $I$ is initialized to be $[0,1]$.
\item For each $1\leq j\leq L^{1/d}-1$, submit the $(j+1)$'th guess at follows: if none of the previous guesses is correct, then inspect the gradient $(f^*_i)'(q_{2j-1,i})$ from the $j$'th guess to deduce which half of $I$ contains the minimizer $X_i^*$ of $f^*_i$. Update the learner's interval $I$ accordingly so that it contains $X_i^*$. Submit the $(j+1)$'th guess at the midpoint of the updated $I$. 
If one of the first $j$ guesses is correct, then update $I$ to its right half, and submit the $(j+1)$'th guess at its midpoint.
\end{enumerate}
When $\delta> 2^{-L^{1/d}}$, only the first $K$ guesses are submitted along the bisection path, and the remaining $L^{1/d}-K$ guesses are submitted via a grid search on the interval $I$ generated from the first $K$ guesses. Here $K$ is the largest integer for which all the guesses are $\delta$-apart. Under the assumption $\delta\leq L^{-1/d}$ such a $K$ always exists.

After all the guesses are submitted, if none of the guesses is correct, the learner runs a simple bisection search on a $\max\{2^{-L^{1/d}},\delta\}$-length interval until reaching $\epsilon$-accuracy; otherwise the learner simply fills the remaining queries along this dimension with trivial queries $q_{i,j}=1$ for all $j\geq 2L^{1/d}$. The total number of queries is exactly the desired upper bound $2L^{1/d} + \log(\max\{2^{-L^{1/d}},\delta\}/\epsilon)$.

Next we show this querying strategy is $(\delta,L)$-private. Here we give the proof in the $\delta\leq 2^{-L^{1/d}}$ case. The proof for the $\delta>2^{-L^{1/d}}$ case follows analogously. 
For each $i$, it is easy to see that if 
\[
X_i^*\in \cup_{j\leq L^{1/d}} [1- 2^{-j}, 1-2^{-j}+\epsilon]
\]
then the queries along the $i$'th dimension would always be $L$ guesses at $1/2,3/4,...,1-2^{L^{1/d}}$, followed by trivial queries at 1. 
As a result, for all $f^*\in \mathcal{F}$ such that
\[
X^*\in \prod_{i\leq d}\left(\cup_{j\leq L^{1/d}}[1- 2^{-j}, 1-2^{-j}+\epsilon]\right)\stackrel{\Delta}{=}J,
\]
share the same query sequence. Clearly $J$ contains $(L^{1/d})^d$ members that are separated by at least $\delta$ in $\|\cdot\|_\infty$ distance. Hence the strategy is $(\delta,L)$-private.

\end{proof}

\begin{proof}[Proof of the lower bound.]
Let $\phi$ be a querying strategy that is $\epsilon$-accurate and $(\delta,L)$-private. 
Via the same argument in one-dimension, we can show that there is at least one query sequence $q$ whose information set $\mathcal{I}(q)$ has a $\delta/2$-covering number at least $L$. For each $i=1,...,d$, let 	
\[
\mathcal{I}_i(q)=\left\{x_i: x=(x_1,...,x_i,...,x_d)\in \mathcal{I}(q) \text{ for some }x\in [0,1]^d\right \}
\]
be the projection of $\mathcal{I}(q)$ to dimension $i$. Then we have $\mathcal{I}(q)\subset \prod_{i\leq d}\mathcal{I}_i(q)$, thus
\[
L\leq N_c\left(\mathcal{I}(q),\delta/2,\|\cdot\|_\infty\right)\leq N_c\left(\prod_{i\leq d} \mathcal{I}_i(q),\delta/2,\|\cdot\|_\infty\right)=\prod_{i\leq d}N_c\left(\mathcal{I}_i(q),\delta/2,|\cdot|\right).
\]
Therefore for at least one $i\leq d$, we must have that the $\delta/2$-covering number of the projection $\mathcal{I}_i(q)$ is no less than $L^{1/d}$. It follows that $\mathcal{I}_i(q)$ contains $x_i^{(1)},...,x_i^{(L^{1/d})}$ that are at least $\delta/2$-apart. For the strategy to be $\epsilon$-accurate, the queries in $q$ along this dimension $i$ must contain at least $L^{1/d}$ pairs of $\epsilon$-apart queries sandwiching $x_i^{(1)},...,x_i^{(L^{1/d})}$. 
The rest of the proof exactly follows the one-dimensional case.
\end{proof}

\section{Acknowledgment}
The authors thank Niva Ran and Benjamin Ran for inspiring the algorithm used in the upper bound of the Bayesian formulation of the problem.

\bibliographystyle{plain}
\bibliography{bib}

\appendix

\section{Self-similarity property of the Dirichlet Process}\label{sec:self.similar}
\begin{proposition}\label{prop:self.similar}
Let $\mu$ be a random probability measure on $\mathcal{X}$ that follows a Dirichlet Process with base distribution function $\mu_0$ and concentration parameter $\alpha$. Let $\mathcal{X}=\cup_{i\leq n}B_i$ be an arbitrary finite partition of $\mathcal{X}$. Then for all $i\leq n$, we have
\[
\mu_{B_i}\mid \mu(B_1), ...,\mu(B_n) \sim \text{DP}\left(\mu_{0,B_i}, \alpha \mu_0\left(B_i\right)\right),
\]
where $\mu_{B_i}$ and $\mu_{0,B_i}$ denote the conditional probability measures of $\mu$ and $\mu_0$ respectively, conditioned on $B_i$.
\end{proposition}
\begin{proof}
For simplicity we present the proof only for $i=1$. The proof for general $i$ is identical. Let $B_1=\cup_{j\leq m}A_j$ be an arbitrary finite partition of $B_1$. Then $(A_1,...,A_m,B_2,...,B_n)$ is a partition of $\mathcal{X}$. Therefore from the definition of the Dirichlet Process, we have
\[
\left(\mu\left(A_1\right),...,\mu\left(A_m\right), \mu\left(B_2\right),...,\mu\left(B_n\right)\right) \sim \text{Dir}\left(\alpha\mu_0\left(A_1\right),...,\alpha\mu_0\left(A_m\right), \alpha\mu_0\left(B_2\right),...,\alpha\mu_0\left(B_n\right)\right).
\]
From the density function of the Dirichlet distribution, we can derive that
\[
\frac{\left(\mu\left(A_1\right),...,\mu\left(A_m\right)\right)}{1-\sum_{i\geq 2}\mu(B_i)}
\;\biggr\rvert\; \mu\left(B_2\right),...,\mu\left(B_n\right)
\sim \text{Dir}\left(\alpha\mu_0\left(A_1\right),...,\alpha\mu_0\left(A_m\right)\right).
\]
Again by definition of the Dirichlet Process, we have
\[
\mu_{B_1}\mid \mu\left(B_2\right),...,\mu\left(B_n\right)\sim \text{DP}\left(\left[\mu_0\right]_{B_1},\alpha\right)=\text{DP}\left(\mu_{0,B_1},\alpha\mu_0\left(B_1\right)\right),
\]
where $[\mu_0]_{B_1}$ denotes the measure $\mu_0$ restricted to $B_1$, which is not necessarily a probability measure.
\end{proof}

Consider the special case where $\mathcal{X}=[0,1]$. 
As a corollary of Proposition~\ref{prop:self.similar}, we have for any finite partition $0=x_0\leq x_1\leq...\leq x_{n-1}\leq x_n=1$ of $[0,1]$, 
\[
\mathcal{L}\left(\frac{[F]_{[x_i,x_{i+1}]}-t_i}{t_{i+1}-t_i} \;\Big\rvert\; F(x_1)=t_1, ...,F(x_{n-1})=t_{n-1}\right)
=\text{DP}\left(\mu_{0,[x_{i},x_{i+1}]}, \alpha \mu_0\left[x_i,x_{i+1}\right]\right).
\]

\section{Proof of Lemma~\ref{lmm:beta}}
In this section we prove the technical result Lemma~\ref{lmm:beta} on the Beta distribution. The statement of Lemma~\ref{lmm:beta} is repeated below.

\begin{replemma}{lmm:beta}
Suppose $X\sim \Beta(\alpha t,\alpha(1-t))$ for some $\alpha>0$, then for all $t\in (0,1)$,
\[
h_\alpha \leq \frac{d}{dt}\mathbb{P}\left\{X\geq 1/2\right\} \leq H_\alpha,
\]
where $h_\alpha = \tfrac{1}{3}2^{-\alpha-2}$ and $H_\alpha= (3+2e^{-1})\alpha+14$.
\end{replemma}

\begin{proof}
We can assume WOLG that $t\in (0,1/2]$. That is because for $t>1/2$, $1-X\sim \Beta(\alpha(1-t),\alpha t)$ and
\[
\frac{d}{dt}\mathbb{P}\left\{X\geq 1/2\right\} = \frac{d}{d(1-t)}\mathbb{P}\left\{1-X\geq 1/2\right\}.
\]
Let $\phi_t(x)=x^{\alpha t-1}(1-x)^{\alpha(1-t)-1}$ be the unnormalized density of the Beta$(\alpha t,\alpha(1-t))$ distribution. 
Since $ \frac{d}{dt}\phi_t(x) = \alpha \ln\frac{x}{1-x}\phi_t(x)$, 
we have
\begin{align*}
\frac{d}{dt}\mathbb{P}\left\{X\geq 1/2\right\} = & \frac{d}{dt}\frac{\int _{1/2}^1 \phi_t(x)dx}{\int_0^1 \phi_t(x)dx}\\
=  & \alpha \frac{\int_{1/2}^1 \ln\frac{x}{1-x}\phi_t(x)dx \int_0^1 \phi_t(x)dx - \int_{1/2}^1 \phi_t(x)dx \int_0^1 \ln\frac{x}{1-x}\phi_t(x)dx}{\left(\int_0^1 \phi_t(x)dx\right)^2}\\
= & \alpha\left[\mathbb{E}\left(\mathds{1}\{X\geq 1/2\}\ln\frac{X}{1-X}\right) - \mathbb{P}\{X\geq 1/2\}\mathbb{E}\left(\ln\frac{X}{1-X}\right)\right].
\end{align*}

To prove the lemma, we claim that for $t\leq 1/2$,
\begin{equation}
\label{eq:exp.prod}
2^{-\alpha-2}t\leq\alpha \mathbb{E}\left(\mathds{1}\{X\geq 1/2\}\ln\frac{X}{1-X}\right) \leq \max\{3\alpha,12\};
\end{equation}
\begin{equation}
\label{eq:prod.exp}
\left[2^{-\alpha-2}\left(\tfrac{1}{2}-\tfrac{t}{1-t} \right)\right]_+\leq - \alpha\mathbb{P}\{X\geq 1/2\}\mathbb{E}\left(\ln\frac{X}{1-X}\right)\leq 2e^{-1}\alpha+2,
\end{equation}
where $[\cdot]_+=\max\{\cdot,0\}$ stands for the positive part.

The upper bound $\frac{d}{dt} \mathbb{P}\{X \ge 1/2\} \leq H_\alpha$ follows easily from adding up the two upper bounds. 
For the lower bound on the derivative, the two lower bounds in~\eqref{eq:exp.prod} and~\eqref{eq:prod.exp} yield
\[
\frac{d}{dt}\mathbb{P}\left\{X\geq 1/2\right\} \geq 2^{-\alpha -2}\left(t+\left(\frac{1}{2}-\frac{t}{1-t}\right)_+\right) \geq \frac{1}{3}2^{-\alpha -2}=h_\alpha,
\]
where the last equality is achieved at $t=1/3$.{}

It remains to prove~\eqref{eq:exp.prod} and~\eqref{eq:prod.exp}. 
Let us start from the cross-product term~\eqref{eq:exp.prod}. Since $\mathds{1}\{X\geq 1/2\}\ln\frac{X}{1-X}\geq 0$, by Tonelli's theorem,
\[
\mathbb{E}\left(\mathds{1}\{X\geq 1/2\}\ln\frac{X}{1-X}\right) = \int_0^\infty \mathbb{P}\left\{\mathds{1}\{X\geq 1/2\}\ln\frac{X}{1-X}>s\right\}ds =\int_0^\infty \mathbb{P}\left\{X\geq \frac{e^s}{1+e^s}\right\}ds.
\]
The density function of $X$ allows us to write
\begin{align}
\label{eq:cross.ratio}
\mathbb{E}\left(\mathds{1}\{X\geq 1/2\}\ln\frac{X}{1-X}\right)= \frac{\int_0^\infty\int_{\frac{e^s}{1+e^s}}^1 x^{\alpha t-1}(1-x)^{\alpha(1-t)-1}dxds}{B(\alpha t,\alpha(1-t))},
\end{align}
where $B(\alpha,\beta)=\int_{0}^1 s^{\alpha-1} (1-s)^{\beta-1} ds$ is the Beta function. First we prove the upper bound in~\eqref{eq:exp.prod}. For the numerator, since $\alpha t-1> -1$ and $x\geq \frac{e^s}{1+e^s}\geq 1/2$, we have $x^{\alpha t-1}\leq 2$, and
\[
\int_{\frac{e^s}{1+e^s}}^1 x^{\alpha t-1}(1-x)^{\alpha(1-t)-1}dx \leq 2\int_{\frac{e^s}{1+e^s}}^1 (1-x)^{\alpha(1-t)-1}dx=\frac{2(1+e^{s})^{-\alpha(1-t)}}{\alpha(1-t)}.
\]
Therefore the numerator of~\eqref{eq:cross.ratio} is upper bounded by
\[
2\int_0^\infty \frac{e^{-\alpha(1-t)s}}{\alpha(1-t)}ds =\frac{2}{ \alpha^2 (1-t)^2 } \leq \frac{8}{\alpha^2}
\]
for all $t\leq 1/2$. Moreover,
\[
B(\alpha t,\alpha(1-t)) =\frac{\Gamma(\alpha t)\Gamma(\alpha (1-t))}{\Gamma(\alpha)} 
\]
is minimized at $t=1/2$ by the log-convexity of the Gamma function $\Gamma(z)$~\cite{artin2015gamma},
where $\Gamma(z)=\int_0^\infty s^{z-1} e^{-s} ds$ satisifying $\Gamma(z+1)=z\Gamma(z)$ for $z>0.$ Hence it follows from~\eqref{eq:cross.ratio}
that for all $t\leq 1/2$,
\begin{equation}
\label{eq:cross.upper}
\alpha\mathbb{E}\left(\mathds{1}\{X\geq 1/2\}\ln\frac{X}{1-X}\right) \leq \frac{8\Gamma(\alpha)}{\alpha\Gamma(\alpha/2)^2}.
\end{equation}
We claim that the right-hand side of~\eqref{eq:cross.upper} is a non-decreasing function in $\alpha$ on $(0,\infty)$. To see that, let $g(\alpha)=8\Gamma(\alpha)/(\alpha \Gamma(\alpha/2)^2)$. We have
\begin{equation}
\label{eq:diff.log.g}
\frac{d}{d\alpha}(\ln g(\alpha))=\frac{\Gamma'(\alpha)}{\Gamma(\alpha)}-\frac{1}{\alpha}-\frac{\Gamma'(\alpha/2)}{\Gamma(\alpha/2)}=\psi(\alpha)-\psi(\alpha/2)-\frac{1}{\alpha}.
\end{equation}
Here $\psi(\cdot)=\Gamma'(\cdot)/\Gamma(\cdot)$ is the digamma function with expansion~\cite[6.3.16]{abramowitz1948handbook}
\[
\psi(1+z)=-\gamma+\sum_{n=1}^\infty\frac{z}{n+z},
\]
where $\gamma$ is the Euler-Mascheroni constant. Applying the expansion on~\eqref{eq:diff.log.g} yields
\[
\frac{d}{d\alpha}\left(\ln g(\alpha) \right)=\sum_{n=1}^\infty\left(\frac{\alpha-1}{n+\alpha-1}-\frac{\alpha/2-1}{n+\alpha/2-1}\right)-\frac{1}{\alpha}\geq \frac{\alpha-1}{1+\alpha-1}-\frac{\alpha/2-1}{1+\alpha/2-1}-\frac{1}{\alpha}=0.
\]
We have shown that $g$ is a non-decreasing function on $\mathbb{R}^+$. It follows from~\eqref{eq:cross.upper} that for all $\alpha\leq 4$, $\alpha\mathbb{E}(\mathds{1}\{X\geq 1/2\}\ln\frac{X}{1-X})\leq g(4)=12$. 

Next we show that for all $\alpha>4$, the cross-product term in~\eqref{eq:exp.prod} is upper bounded by $3\alpha$. By Markov's inequality,
\[
\mathbb{P}\left\{X\geq \frac{e^s}{1+e^s}\right\}
=\mathbb{P}\left\{1-X\leq \frac{1}{1+e^s}\right\}
=\mathbb{P}\left\{ \frac{1}{1-X} \geq 1+e^s \right\}
\leq \frac{1}{1+e^s}\mathbb{E}\left[\frac{1}{1-X}\right].
\]
Since $1-X\sim \Beta(\alpha(1-t),\alpha t)$, we have
\[
\mathbb{E}\left[\frac{1}{1-X} \right]=\frac{\int_0^1 x^{\alpha(1-t)-2}(1-x)^{\alpha t-1}dx}{\int_0^1 x^{\alpha(1-t)-1}(1-x)^{\alpha t-1}dx}.
\]
For all $\alpha\geq 4$ and $t\leq 1/2$, $\alpha(1-t)-1\geq 0$, hence both integrals converge, and
\[
\mathbb{E}\left[\frac{1}{1-X} \right]=\frac{B(\alpha(1-t)-1,\alpha t)}{B(\alpha(1-t),\alpha t)}
=\frac{\Gamma(\alpha(1-t)-1)\Gamma(\alpha t)/\Gamma(\alpha-1)}{\Gamma(\alpha(1-t))\Gamma(\alpha t)/\Gamma(\alpha)}
=\frac{\alpha-1}{\alpha(1-t)-1}\leq 3
\]
when $\alpha\geq 4$. Therefore
\[
\alpha \mathbb{E}\left(\mathds{1}\{X\geq 1/2\}\ln\frac{X}{1-X}\right)\leq 3\alpha\int_0^\infty \frac{1}{1+e^s}ds \leq 3\alpha.
\]
That finishes the proof of the upper bound in~\eqref{eq:exp.prod}. Next we prove the lower bound in~\eqref{eq:exp.prod}.
Since $x^{\alpha t-1}\geq \min\{(1/2)^{\alpha t-1},1\}$ for all $x\geq e^s/(1+e^s)\geq 1/2$, we have that the numerator in~\eqref{eq:cross.ratio} is lower bounded by
\begin{align}
\nonumber & \min\left\{\left(\tfrac{1}{2}\right)^{\alpha t-1}, 1\right\}\int_0^\infty\int_{\frac{e^s}{1+e^s}}^1(1-x)^{\alpha(1-t)-1} dxds\\
\nonumber = & \frac{\min\left\{\left(\frac{1}{2}\right)^{\alpha t-1}, 1\right\}}{\alpha(1-t)}\int_0^\infty \left(\frac{1}{1+e^s}\right)^{\alpha(1-t)} ds\\
\nonumber \geq & \frac{\min\left\{\left(\frac{1}{2}\right)^{\alpha t-1}, 1\right\}\left(\frac{1}{2}\right)^{\alpha(1-t)}}{\alpha(1-t)}\int_0^\infty e^{-s\alpha(1-t)}ds\\
\label{eq:cross.ratio.num.lower}= & \frac{\left(\frac{1}{2}\right)^{\max\{\alpha-1,\alpha(1-t)\}}}{\alpha^2(1-t)^2}\geq \frac{2^{-\alpha}}{\alpha^2}.
\end{align}
To handle the denominator in~\eqref{eq:exp.prod}, note that $(1-x)^{\alpha(1-t)-1}\leq 2$ for all $x\leq 1/2$ and $x^{\alpha t-1}\leq 2$ for all $x\geq 1/2$. Therefore the denominator in~\eqref{eq:exp.prod}
\begin{equation}
\label{eq:cross.ratio.den.upper}
B(\alpha t,\alpha (1-t))\leq 2 \int_0^{1/2}x^{\alpha t-1}dx+2\int_{1/2}^1(1-x)^{\alpha(1-t)-1}dx
=  2\left[\frac{2^{-\alpha t}}{\alpha t}+\frac{2^{-\alpha(1-t)}}{\alpha(1-t)}\right] \leq \frac{2}{\alpha t(1-t)}.
\end{equation}
Combining~\eqref{eq:cross.ratio},~\eqref{eq:cross.ratio.num.lower} and~\eqref{eq:cross.ratio.den.upper} yields
\[
\alpha \mathbb{E}\left(\mathds{1}\{X\geq 1/2\}\ln\frac{X}{1-X}\right) \geq \alpha\frac{2^{-\alpha} \alpha t(1-t)}{2\alpha^2}\geq 2^{-\alpha-2} t.
\]

Next let us prove~\eqref{eq:prod.exp}. Firstly, write
\[
\mathbb{E}\left(\ln\frac{X}{1-X}\right)=\psi(\alpha t)-\psi(\alpha)-\left(\psi(\alpha(1-t))-\psi(\alpha)\right)=\psi(\alpha t)-\psi(\alpha (1-t))
\]
where we recall that $\psi(z)=\frac{d}{dz}\ln \Gamma(z)$ is the digamma function. Since $\Gamma$ is log-convex on $\mathbb{R}^+$, $\psi$ is non-decreasing. Therefore for all $t\leq 1/2$, we have
\[
-\alpha \mathbb{P}\{X\geq 1/2\}\mathbb{E}\left(\ln\frac{X}{1-X}\right)\geq 0.
\]
Furthermore, it has been shown in~\cite[Eq~(2.2)]{alzer1997some} that for all $z>0$, the digamma function satisfies
\begin{equation}
\label{eq:psi.bound}
\frac{1}{2z}<\ln z-\psi(z)<\frac{1}{z}.
\end{equation}
Therefore
\begin{equation}
\label{eq:exp.log.lower}
-\mathbb{E}\left(\ln\frac{X}{1-X}\right)=\psi(\alpha(1-t))-\psi(\alpha t)\geq \ln(\alpha(1-t))-\frac{1}{\alpha(1-t)}-\ln(\alpha t)+\frac{1}{2\alpha t}\geq \frac{1}{\alpha}\left(\frac{1}{2t}-\frac{1}{1-t}\right)
\end{equation}
when $t\leq 1/2$. 

We still need to bound $\mathbb{P}\{X\geq 1/2\}$ from below. As in the proof of~\eqref{eq:exp.prod}, we can write
\begin{equation}
\label{eq:tail.ratio}
\mathbb{P}\left\{X\geq \tfrac{1}{2}\right\}=\frac{\int_{1/2}^1 x^{\alpha t-1}(1-x)^{\alpha (1-t)-1}dx}{B(\alpha t,\alpha (1-t))}.
\end{equation}
Again from $x^{\alpha t-1}\geq \min\{ (1/2)^{\alpha t-1}, 1\}$ for all $x\geq 1/2$, we have that the numerator of~\eqref{eq:tail.ratio} is bounded from below by
\[
\max\left\{ \left(\tfrac{1}{2}\right)^{\alpha t-1},1\right\}\int_{1/2}^1 (1-x)^{\alpha (1-t)-1}dx\geq \frac{2^{-\alpha}}{\alpha}.
\]
Combining the last displayed equation with~\eqref{eq:cross.ratio.den.upper} and~\eqref{eq:tail.ratio} yields that 
$$
\mathbb{P}\left\{X\geq \tfrac{1}{2}\right\} \ge  \frac{2^{-\alpha}}{\alpha} \times \frac{\alpha t(1-t)}{2}  \ge 2^{-\alpha-2} t
$$
for all $t \le 1/2$.
In view of~\eqref{eq:exp.log.lower}, it follows that 
\[
-\alpha\mathbb{P}\left\{X\geq \tfrac{1}{2}\right\}\mathbb{E}\left(\ln\frac{X}{1-X}\right)\geq 2^{-\alpha -2}\left(\frac{1}{2}-\frac{t}{1-t}\right).
\]
That concludes the proof of the lower bound in~\eqref{eq:prod.exp}. Next we move to the upper bound in~\eqref{eq:prod.exp}.
By Markov's inequality,
\begin{equation}
\label{eq:markov}
\mathbb{P}\{X\geq 1/2\}\leq 2\mathbb{E}X=2t.
\end{equation}
Again from~\eqref{eq:psi.bound} we have that for all $t\leq 1/2$,
\begin{align*}
- \mathbb{E}\left(\ln\frac{X}{1-X}\right) 
=& \psi(\alpha(1-t))-\psi(\alpha t)\\
\leq & \ln(\alpha(1-t))-\frac{1}{2\alpha(1-t)}-\left(\ln(\alpha t)-\frac{1}{\alpha t}\right)\\
 = & \ln\frac{1-t}{t}+\frac{2-3t}{2\alpha t(1-t)}.
\end{align*}
Combining the last displayed equation with~\eqref{eq:markov} yields that
\begin{align*}
- \alpha \mathbb{P}\{X\geq 1/2\}\mathbb{E}\left(\ln\frac{X}{1-X}\right) \leq & 2\alpha t\left(\ln\frac{1-t}{t}+\frac{2-3t}{2\alpha t(1-t)}\right)\\
\leq& (2t\ln (1/t))\alpha + \frac{2-3t}{1-t}\leq  2e^{-1}\alpha +2.
\end{align*}
We have thus established the inequalities~\eqref{eq:exp.prod} and~\eqref{eq:prod.exp}. 
\end{proof}

\newpage

\end{document}